\documentclass[10pt]{article} 
\usepackage[preprint]{rlc}

\usepackage{amssymb,amsthm}            
\usepackage{tikz}
\usetikzlibrary{patterns, calc, decorations.pathreplacing} 

\usepackage{mathtools}          
\usepackage{mathrsfs}           
\mathtoolsset{showonlyrefs}     
\usepackage{graphicx}           
\usepackage{subcaption}         
\usepackage[space]{grffile}     
\usepackage{url}                
\let\classAND\AND
\let\AND\relax
\usepackage{algorithmic}

\let\AND\classAND
\AtBeginEnvironment{algorithmic}{\let\AND\algoAND}
\usepackage[ruled, vlined]{algorithm2e}

\SetCommentSty{mycommfont}
\usepackage{float}
\usepackage{booktabs}

\newtheorem{assumption}{Assumption}
\newtheorem{remark}{Remark}

\newtheorem{theorem}{Theorem}
\newtheorem{lemma}{Lemma}
\newtheorem{definition}{Definition}
\newtheorem{proposition}{Proposition}
\newtheorem{corollary}{Corollary}
\newcommand{\red}[1]{\textcolor{red}{#1}}

\renewcommand{\hat}{\widehat}

\title{Cost Aware Best Arm Identification}

\author{Kellen Kanarios \\
    kellenkk@umich.edu \\
    University of Michigan, Ann Arbor 
    \And
    Qining Zhang \\
    qiningz@umich.edu\\
    University of Michigan, Ann Arbor 
    \AND
    Lei Ying \\
    leiying@umich.edu \\
    University of Michigan, Ann Arbor
    }

\begin{document}

\maketitle

\begin{abstract}
In this paper, we study a best arm identification problem with dual objects. In addition to the classic reward, each arm is associated with a cost distribution and the goal is to identify the largest reward arm using the minimum expected cost. We call it \emph{Cost Aware Best Arm Identification} (CABAI), which captures the separation of testing and implementation phases in product development pipelines and models the objective shift between phases, i.e., cost for testing and reward for implementation. We first derive a theoretical lower bound for CABAI and propose an algorithm called $\mathsf{CTAS}$ to match it asymptotically. To reduce the computation of $\mathsf{CTAS}$, we further propose a simple algorithm called \emph{Chernoff Overlap} (CO), based on a square-root rule, which we prove is optimal in simplified two-armed models and generalizes well in numerical experiments. Our results show that (i) ignoring the heterogeneous action cost results in sub-optimality in practice, and (ii) simple algorithms can deliver near-optimal performance over a wide range of problems.
\end{abstract}

\section{Introduction}

The stochastic multi-armed bandit (MAB)~\citep{thompson1933likelihood, OriginalMAB} is a classic model which has widespread applications, from content recommendation~\citep{kohli2013recommendation}, resource allocation~\citep{liu2020pond}, clinical trials~\citep{villar2015clinic}, to efficient ad placement. A multi-armed bandit problem involves an agent and an environment, which is represented by a set of actions (arms) with distinct underlying reward distributions. At each round, the agent will choose one of the arms and then obtain a random reward generated from the associated distribution. Most existing studies formulate MAB either as the best arm identification (BAI) problem~\citep{kaufmann2016complexity, Garovoer16BAI} where the agent intends to identify the highest reward arm as quickly as possible or as a regret minimization problem~\citep{auer2002ucb,garivier2011kl} where the goal is to maximize the cumulative reward over a time-horizon. Both formulations have been well-developed and successful in balancing the trade-off between exploration and exploitation. 

However, unlike the classic MAB model, most real-world product development pipelines are usually separated into two phases: testing (survey) and implementation (release). Here, testing refers to the process where one intends to find the best product among a set of potential candidates through sequential trials, e.g., A/B testing for clinical decisions. The implementation phase refers to the selected best product being used in a wider population after testing, usually involving mass production. Different performance measures are emphasized in different phases. For example, the cost of prototype medicine may be of primary concern in testing while the efficacy is more important in implementation since the production cost is either decreased via mass production or covered through insurance. Similarly, when choosing the best platform for online advertising, the total payment to different platforms for advertising may be of essential concern in testing, while the click-through and conversion rate are what matters in implementation after the best platform is selected. Unfortunately, neither BAI nor regret minimization captures the aforementioned differences, which makes algorithms developed for traditional MAB not directly applicable. Moreover, trying out different candidates during testing may require different costs, which again is not captured in classic MAB.

\textbf{Cost Aware Best Arm Identification:} in this paper, we propose a new MAB problem called \emph{Cost Aware Best Arm Identification} (CABAI), where besides reward (the main object for implementation), each arm is in addition associated with a cost distribution to model the object for testing. Each time the agent chooses an arm, it will observe a random reward-cost pair, which is independently generated from the reward distribution and cost distribution respectively. The goal in CABAI is to identify the highest reward arm using the minimum cost possible, which breaks down to the following questions: (1) how should we sample arms during testing with unknown cost, and how does the rule differ from BAI (\textbf{sampling rule})? (2) when is the best arm identifiable (\textbf{stopping rule})? (3) which arm should we choose for implementation (\textbf{decision rule})? We will show that the design of algorithms for CABAI is related to BAI, but have fundamental differences so that BAI optimal algorithms do not necessarily achieve good performance in CABAI. This also implies that directly applying BAI algorithms and neglecting the heterogeneous nature of arms in practice will result in sub-optimality. We address the following questions in our paper: (1) What are the fundamental limits of CABAI? (2) How should we design efficient algorithms to achieve the limit?

\textbf{Our Contributions:} We propose CABAI and show that traditional BAI algorithms no longer perform well. As summarized in Table \ref{tab:optimal-pulls}, the optimal proportions of arm pulls have essential differences between traditional BAI and CABAI, i.e., $\mathsf{TAS}$~\citep{Garovoer16BAI}, which is optimal in BAI, allocates almost the same amount of pulls to the first two arms, while the optimal proportion of arm pulls for CABAI emphasize more on the low-cost arm, as achieved by our proposed $\mathsf{CTAS}$ algorithm. We first prove a non-asymptotic lower bound on the minimum cumulative cost required to identify the best arm. Then, we propose an algorithm called \emph{Cost-Aware Track and Stop} ($\mathsf{CTAS}$) to match the lower bound asymptotically. However, the $\mathsf{CTAS}$ algorithm is required to solve a bilevel optimization problem at each time step, which exerts relatively high computational complexity and prevents its direct use in practice. To overcome this issue, we further propose a low-complexity algorithm called \emph{Chernoff Overlap} $(\mathsf{CO})$ which exhibits desirable empirical performance and remains theoretically optimal in simplified bandit models.
\begin{table}[H]
    \centering
    \begin{tabular}{ccccc}
        \toprule
         Algorithm & Optimal? &  $w_1(t)$ $(1.5, 1)$  & $w_2(t)$ $(1, 0.1)$ & $w_3(t)$ $(0.5, 0.01)$ \\
        \midrule
        $\mathsf{TAS}$ & $\times$  & $0.46$ & $0.46$ & $0.08$ \\
        $\mathsf{CTAS}$ & \checkmark  & $0.23$ & $0.72$ & $0.05$ \\
        \bottomrule
    \end{tabular}
    \caption{The expected rewards are $\boldsymbol{\mu} = [1.5, 1, 0.5]$ and the expected costs are $\boldsymbol{c} = [1, 0.1, 0.01]$. For arm $i$, $w_i(t)$ is the proportion of arm pulls up to time $t$, i.e., $w_i(t) = N_i(t)/t$ where $N_i$ is the number of pulls. Noticeably, CABAI emphasizes more on low-cost arms to complement high-cost arms.}
    \label{tab:optimal-pulls}
\end{table}

\subsection{Related Work}
We review existing MAB results most relevant to our paper. A detailed discussion is in the appendix.

\textbf{BAI with Fixed Confidence:} BAI has been studied for many years and was originally proposed in \citet{bechhofer1958sequential}. In this paper, we consider a subset known as the fixed confidence setting, where the agent aims to minimize the sample complexity while ensuring the best arm is identified with probability at least $1 - \delta$. Here, $\delta$ is a pre-specified confidence level, and such algorithms are called $\delta$-PAC. In \citet{kaufmann2016complexity}, the authors introduce a non-asymptotic lower bound for this setting. Subsequently, they propose the Track and Stop algorithm ($\mathsf{TAS}$) that matches this lower bound asymptotically.  The $\mathsf{TAS}$ algorithm has since been extended to various other settings~\citep{pmlr-v201-jourdan23a, garivier2021nonasymptotic, kato2021role}. Before it, researchers proposed ``confidence-based'' algorithms, e.g., $\mathsf{KL}$-$\mathsf{LUCB}$ \citep{Kaufmann13kl}, $\mathsf{UGapE}$ \citep{UGapE}, which rely on constructing high-probability confidence intervals. They are more computationally feasible than $\mathsf{TAS}$ inspired algorithms, but with few theoretical guarantees. 

\textbf{BAI with Safety Constraints:} A formulation similar to our paper is BAI with safety constraints~\citep{wang2022best}. As a motivating example, they consider the clinical trial setting, where each drug is associated with a dosage and the dosage has an associated safety level. They attempt to identify the best drug and dosage for fixed confidence without violating the safety level. Similarly, \citet{hou2022almost} attempts to identify the best arm subject to a constraint on the variance. Our formulation is distinct from them because the agent is free to perform any action. In \cite{chen2023doublyoptimistic, chen2022strategies}, they formulate safety constraints as a constrained optimization problem. They explore and show that allowing minimal constraint violations can provide significant improvement in the regret setting. This is distinct from the BAI setting explored in this paper.

\textbf{Multi-fidelity BAI:} An alternative formulation that considers cost is the multi-fidelity formulation introduced in \cite{kandasamy2016multifidelity} and recently considered in the best arm identification regime \citep{multi-fidelity-2022, wang2023multifidelity}. In this setting, along with choosing an arm, the agent chooses the desired fidelity or ``level of accuracy'' of the mean estimate. Each fidelity incurs a cost, where higher fidelity incurs a larger cost but provides more accurate estimate. This setting clearly differs from ours because the cost of each fidelity is known a priori and is controllable through choice of fidelity.

\section{Preliminaries}

We study a model similar to the fixed-confidence BAI in stochastic $K$-armed bandits. We denote the set of arms as $\mathcal{A} \coloneqq \{1,2,\cdots,K\}$. Each arm $a$ is associated with a reward distribution $\boldsymbol{\nu}_{\mu} = \{\nu_{\mu_1}, \ldots, \nu_{\mu_K}\}$ with expectations $\boldsymbol{\mu} \coloneqq \{\mu_1, \mu_2, \cdots, \mu_K\}$. We assume $\boldsymbol{\nu}_{\mu}$ are independent and make the natural exponential family assumption standard in BAI literature~\citep{kaufmann2016complexity}:
\begin{assumption}[Natural Exponential Family]\label{ass:exp-family}
    For any $a$, $\nu_{\mu_a}$ belongs to family $\mathcal{P}$ which can be parameterized by the expectation with finite moment generating function, i.e.,
    $$
        \mathcal{P}=\{\nu_{\mu}| \mu \in[0,1], \nu_{\mu}= h(x) \exp (\theta_\mu x-b(\theta_\mu)) \},
    $$
    where $\theta_\mu$ is a function of $\mu$, and $b(\theta)$ is convex and twice differentiable. 
\end{assumption}
For two different expectations $\mu$ and $\mu'$ with the same exponential family, we use $d\left(\mu, \mu^{\prime}\right)$ to denote the KL-divergence from $\nu_\mu$ to $\nu_{\mu'}$. Note that Assumption~\ref{ass:exp-family} is very general and includes a large class of distributions such as Gaussian (with known variance), Bernoulli, and Poisson distributions by considering the following choice of parameters:
\begin{alignat*}{3}
        \mathsf{Bernoulli:} \quad \theta_{\mu} &= \log \left(\frac{\mu}{1 - \mu}\right),\quad &&b(\theta_{\mu}) = \log \left(1 + e^{\theta_{\mu}}\right), \quad &&h(x) = 1 \\
        \mathsf{Poisson:} \quad \theta_{\mu} &= \log(\mu), \quad  &&b(\theta_{\mu}) = e^{\theta_{\mu}}, \quad &&h(x) = \frac{1}{x!}e^{-x} \\
        \mathsf{Gaussian:} \quad \theta_{\mu} &= \frac{\mu}{\sigma^2}, &&b(\theta_{\mu}) = \frac{\sigma^2 \theta_{\mu}^2}{2}, &&h(x) = \frac{1}{\sqrt{2\pi}}e^{-x^2/2\sigma^2} 
\end{alignat*}
Unique to this work, we assume that each arm has a cost with distribution $\boldsymbol{\nu}_{c} \coloneqq \{\nu_{c_1}, \ldots, \nu_{c_K}\}$ and expectations $\boldsymbol{c} \coloneqq \{c_1, \ldots, c_K\}$. We assume they satisfy the positivity assumption, which is natural in our motivating examples in real world such as ad placement or clinical trials, where the cost of each action is always bounded and all actions are not free.

\begin{assumption}[Bounded Positivity]\label{ass:boundedness}
    For any arm $a$, we assume the support of the cost distribution $\nu_{c_a}$ is positive and bounded away from $0$, i.e., $\mathsf{supp}(\nu_{c_a}) \in [\ell, 1]$, where $\ell$ is a positive constant. 
\end{assumption}

\textbf{Problem Formulation:} We define the best arm $a^*(\boldsymbol{\mu})$ to be the action which has the highest expected reward, i.e., $a^*(\boldsymbol{\mu}) = \arg\max_{a\in\mathcal{A}}\mu_a$, and we assume there is a unique best arm. The results can be generalized to scenarios with multiple best arms given the number of best arms. 
At each round (time) $t \in \mathbb{N}^+$, we interact with the environment by choosing an arm $A_t\in \mathcal{A}$. After that, a (reward, cost) signal pair $(R_t, C_t)$ is independently sampled from the joint distribution $\nu_{\mu_{A_t}} \times \nu_{c_{A_t}}$ of the action that we choose. 
For any time $t$, we use $N_a(t)$ to denote the number of times that arm $a$ has been pulled, and we use $\widehat{\mu}_a(t)$ and $\hat{c}_a(t)$ to denote the empirical average reward and cost: 
\begin{align*}
    \widehat{\mu}_a(t) &= \frac{1}{N_a(t)} \sum_{k = 1}^{t} R_k \cdot \mathbf{1}_{\{A_k = a\}},\quad \widehat{c}_a(t) = \frac{1}{N_a(t)} \sum_{k = 1}^{t} C_k \cdot \mathbf{1}_{\{A_k = a\}}.
\end{align*}

For any policy $\pi$, it consists of three components: (1) a sampling rule $(A_t)_{t\geq 1}$ to select arms to interact at each round; (2) a stopping time $\tau_\delta$ which terminates the interaction; and (3) an arm decision rule $\hat{a}$ to identify the best arm. As a convention of BAI with fixed confidence, we require our policy $\pi$ to be \emph{$\delta$-PAC} (Probably Approximately Correct)~\citep{kaufmann2016complexity}, which means the algorithm should terminate in finite time and the probability of choosing the wrong best arm should be lower than the confidence level $\delta$.  The definition of $\delta$-PAC is as follows:

\begin{definition}
    An algorithm $\pi$ is $\delta$-PAC if for any reward and cost instances $(\boldsymbol{\mu},\boldsymbol{c})$, it outputs the best arm $a^*(\boldsymbol{\mu})$ with probability at least $1-\delta$ and in finite time almost surely, i.e.,
    \begin{align}
        \mathbb{P}_{\boldsymbol{\mu} \times \boldsymbol{c}}(\hat{a} \neq a^*(\boldsymbol{\mu})) \leq \delta, 
        \quad 
        \mathbb{P}_{\boldsymbol{\mu} \times \boldsymbol{c}}(\tau_\delta<\infty) = 1.
    \end{align}
\end{definition}
For any time $t$, define the cumulative cost as $J(t) \coloneqq \sum_{k=1}^t C_k$. For any fixed $\delta$, let $\Pi_\delta$ denote the set of all $\delta$-PAC best arm identification policies. The goal of this work is to find $\pi \in \Pi_\delta$, such that $\pi = \arg\min_{\pi\in \Pi_\delta} \mathbb{E}_{\boldsymbol{\mu}\times \boldsymbol{c}}\left[ J(\tau_\delta) \right]$. We use boldface $\boldsymbol{x}$ to denote vectors and instances, and calligraphy $\mathcal{X}$ to denote sets.
We use the subscript $\mathbb{P}_{\boldsymbol{\mu} \times \boldsymbol{c}}$, $\mathbb{E}_{\boldsymbol{\mu} \times \boldsymbol{c}}$ to denote the probability measure and expectation with respect to a specific instance $(\boldsymbol{\mu}, \boldsymbol{c})$.

\section{Lower Bound}
We first characterize the theoretical limits of this cost minimization problem. Denote by $\mathcal{M}$ a set of exponential bandit models such that each bandit model $\boldsymbol{\mu}=\left(\mu_1, \ldots, \mu_K\right)$ in $\mathcal{M}$ has a unique best arm $a^*(\boldsymbol{\mu})$. Let $\Sigma_K=\left\{ \boldsymbol{w} \in \mathbb{R}_{+}^k: w_1+\cdots+w_K=1\right\}$ to be the set of probability distributions on $\mathcal{A}$, then we present the following theorem which characterizes the fundamental lower bound.

\begin{theorem}\label{thm:general-lower-bound}
    Let $\delta\in(0,1)$. For any $\delta$-PAC algorithm and any bandit model $\boldsymbol{\mu} \in \mathcal{M}$, we have:
    \begin{align*}
        \mathbb{E}_{\boldsymbol{\mu}\times \boldsymbol{c}}\left[J(\tau_\delta)\right] \geq T^*(\boldsymbol{\mu}) \log \frac{1}{\delta} + o\left(\log \frac{1}{\delta}\right).
    \end{align*}
    where $T^*(\boldsymbol{\mu})$ is the instance dependent constant satisfying:
    \[
        T^*(\boldsymbol{\mu})^{-1} = \sup_{\boldsymbol{w} \in \Sigma_K}\inf_{\boldsymbol{\lambda} \in \left\{a^*(\boldsymbol{\lambda}) \neq a^*(\boldsymbol{\mu})\right\}}\sum_{a}\frac{w_a}{c_a}d(\mu_a,\lambda_a).
    \]
\end{theorem}

The proof of Theorem \ref{thm:general-lower-bound} is deferred to the appendix but primarily relies on the ``transportation'' lemma proposed in~\citet{kaufmann2016complexity}, which characterizes the theoretical hardness to distinguish the bandit model $\boldsymbol{\mu}$ from any other models $\boldsymbol{\lambda}$ where $a^*(\boldsymbol{\lambda}) \neq a^*(\boldsymbol{\mu})$. Theorem~\ref{thm:general-lower-bound} suggests that $\mathcal{O}(\log (1/\delta))$ cumulative cost is inevitable to identify the optimal arm, and it also characterizes the asymptotic lower bound constant $T^*(\boldsymbol{\mu})$.

\textbf{Instance Dependent Constant $T^*(\boldsymbol{\mu})$: }  The instance dependent constant $T^*(\boldsymbol{\mu})$ obtained in our Theorem.~\ref{thm:general-lower-bound} is different from classic best arm identification lower bounds, e.g., $T^*(\boldsymbol{\mu})$ in Theorem 1 of~\citet{Garovoer16BAI}. Even though it captures the hardness of this instance in terms of the cumulative cost, $T^*(\boldsymbol{\mu})$ is still a vague notion in the sense that the relationship between the theoretical cumulative cost $J(\tau_\delta)$ and model parameters, $\boldsymbol{\mu}$, $\boldsymbol{c}$, is still unclear. To better understand this mysterious constant $T^*(\boldsymbol{\mu})$, we present Theorem.~\ref{thm:general-lower-bound} in the simple case of $2$ armed Gaussian bandits with unit variance, where $T^*(\boldsymbol{\mu})$ has a closed-form expression. 

\begin{corollary}\label{corr:2-arm-gauss}
    Let $\delta\in(0,1)$. For any $\delta$-PAC algorithm and any 2-armed Gaussian bandits with reward expectations $\{\mu_1, \mu_2\}$ and unit variance such that $\mu_1 > \mu_2$ , we have:
    \begin{align*}
        \mathbb{E}_{\boldsymbol{\mu}\times \boldsymbol{c}}\left[J(\tau_{\delta})\right]\geq \frac{2\left(\sqrt{c_1}+\sqrt{c_2}\right)^2 }{(\mu_1 - \mu_2)^2} \log \frac{1}{\delta} + o\left(\log \frac{1}{\delta}\right).
    \end{align*}
\end{corollary}

It is noticeable that the dependence on cost is non-trivial but somehow involves the square root $\sqrt{c_a}$ for each action. This inspires our low-complexity algorithm Chernoff Overlap ($\mathsf{CO}$) based on a square-root rule. The lower bound of a slightly more general setting is provided in the appendix.

\textbf{The Optimal Weight $\boldsymbol{w}^*$ :} Let $\boldsymbol{w}^* = \{w_a^*\}_{a\in\mathcal{A}}$ be the solution of the sup-inf problem in the definition of $T^*(\boldsymbol{\mu})$ in Theorem \ref{thm:general-lower-bound}. The weight $\boldsymbol{w}^*$ is essential in designing efficient algorithms to match the lower bound, as it characterizes the optimal proportion of the total cumulative cost from pulling arm $a$. Concretely, any algorithm which matches the lower bound should satisfy:
\begin{equation}\label{eq:w-star}
    \lim_{\delta \rightarrow 0}\frac{c_a \mathbb{E}_{\boldsymbol{\mu}\times \boldsymbol{c}}\left[N_a(\tau_{\delta})\right]}{\mathbb{E}_{\boldsymbol{\mu}\times \boldsymbol{c}}\left[J(\tau_{\delta})\right]} = w_a^*, \quad \forall a\in\mathcal{A}
\end{equation}
This differs from~\citet{Garovoer16BAI}, where $w_a$ is the proportion of rounds that arm $a$ is pulled. Like $T^*(\boldsymbol{\mu})$, there is no closed-form expression for $\boldsymbol{w}^*$ in general bandit models with $K \geq 3$. In the appendix, we show that one can compute the desired quantities such as $T^*(\boldsymbol{\mu})$ and $\boldsymbol{w}^*$ by similarly solving $K$ continuous equations to \citet{Garovoer16BAI}. Therefore, we can readily apply iterative methods such as bisection to compute these values. We summarize this procedure in the $\mathsf{ComputeProportions}$ Algorithm (Algorithm.~\ref{alg:opt-prop} in the Appendix), which will be called regularly as a sub-routine in our proposed algorithms (Algorithm.~\ref{alg:ctas}). 

\section{Asymptotically Cost Optimal Algorithm}

\IncMargin{1em}
\begin{algorithm}[t]
    \caption{\textbf{C}ost-adapted \textbf{T}rack \textbf{A}nd \textbf{S}top ($\mathsf{CTAS}$)}\label{alg:ctas}
    \KwIn{confidence $\delta$; $\alpha \geq 1$; sufficiently large $B$; oracle function $\mathsf{ComputeProportions}(\boldsymbol{\mu},\boldsymbol{c})$.}
     pull each arm $a\in\mathcal{A}$ once as initialization\;
     \For{$t\geq K+1$}{
     forced exploration set $\mathcal{U}_t = \{a \mid N_a(t) < \sqrt{t}\}$ \;
     $\boldsymbol{w}^* = \mathsf{ComputeProportions}(\hat{\boldsymbol{\mu}}(t), \hat{\boldsymbol{c}}(t))$ \tcp*{compute optimal proportion}
    \If(\tcp*[f]{Sampling Rule}){$\mathcal{U}_t \neq \emptyset$}{
         pull the least-pulled arm: $a_{t} \in \underset{a \in \mathcal{A}}{\operatorname{argmin}}\ N_a(t)$ 
     }
    \Else{
        $a_{t} \in \underset{a \in \mathcal{A}}{\operatorname{argmax}}\ J(t)w_a^*- \widehat{c}_a(t) N_a(t)$ \tcp*{pull the arm with largest deficit}
        }
        \If(\tcp*[f]{Stopping Rule}){$Z(t) > \log\left(\frac{Bt^{\alpha}}{\delta}\right)$}{
            \textbf{break};
        }
    }
    \Return $\hat{a} = \underset{a \in \mathcal{A}}{\operatorname{argmax}}\  \hat{\mu}_a(t)$ \tcp*{Decision Rule}
\end{algorithm}

In this section, we propose a BAI algorithm called \textbf{C}ost-aware \textbf{T}rack \textbf{A}nd \textbf{S}top ($\mathsf{CTAS}$) whose cumulative cost performance asymptotically matches the lower bound in Theorem~\ref{thm:general-lower-bound} both in expectation and almost surely. We discuss each of the sampling, stopping and decision rules for $\mathsf{CTAS}$:

\textbf{Sampling Rule:}
From \eqref{eq:w-star}, a necessary condition for the optimal algorithm is derived. Our sampling rule in Algorithm \ref{alg:ctas} strives to match the proportion of the cost of each arm to the optimal proportion $\boldsymbol{w}^*(\boldsymbol{\mu})$. First, we force the empirical proportions $\hat{w}_{a} = \hat{c}_a N_a(t)/{J(t)}$ to not differ too greatly from the empirically optimal weights $\boldsymbol{w}^*(\hat{\boldsymbol{\mu}})$ using a largest-deficit-first like arm selection policy. We will show that as the empirical mean $\hat{\boldsymbol{\mu}} \to \boldsymbol{\mu}$, we will have $\boldsymbol{w}^*(\hat{\boldsymbol{\mu}}) \to \boldsymbol{w}^*(\boldsymbol{\mu})$, and the empirical cost proportion will also converge and concentrate along the optimal proportion $\boldsymbol{w}^*(\boldsymbol{\mu})$.

\begin{remark}
For Theorem \ref{thm:as-conv} and the almost sure optimality result, we do not need such a fast rate of convergence and only require that for every $a$, $N_a(t) \to \infty$.
\end{remark}

\textbf{Forced Exploration:} Also present in Algorithm \ref{alg:ctas} is the forced exploration, which pulls the least-pulled arm when $\mathcal{U}(t)$ is not empty (Line 6). This ensures each arm is pulled at least $\Omega(\sqrt{t})$ times, and makes sure that our plug-in estimate of $\boldsymbol{w}^*$ is sufficiently accurate. The $\sqrt{t}$ rate of forced exploration is carefully chosen to balance the sample complexity and the convergence rate of the empirical mean. If chosen too small, the fraction of cost from different arms will concentrate along the inaccurate estimation which results in sub-optimality. If chosen too large, the forced exploration will dominate the sampling procedure, leading to an almost uniform exploration which is sub-optimal.

\textbf{Stopping Rule and Decision Rule:}
We utilize the Generalized Likelihood Ratio statistic \citep{Chernoff1959} between the observations of arm $a$ and arm $b$ $Z_{a, b}(t)$. For an arbitrary exponential family, $Z_{a,b}(t)$ has a closed-form expression as follows:
\begin{align*}
    Z_{a, b}(t)= N_a(t) &d\left(\hat{\mu}_a(t), \hat{\mu}_{a, b}(t)\right) +N_b(t) d\left(\hat{\mu}_b(t), \hat{\mu}_{a, b}(t)\right),
\end{align*}
where $\hat{\mu}_{a, b}(t) = \hat{\mu}_{b, a}(t)$ is defined:
\[
\hat{\mu}_{a, b}(t):=\frac{N_a(t)}{N_a(t)+N_b(t)} \hat{\mu}_a(t)+\frac{N_b(t)}{N_a(t)+N_b(t)} \hat{\mu}_b(t).
\]
In particular, the Chernoff statistics $Z(t) = \max_{a \in \mathcal{A}} \min_{b \in\mathcal{A}, b \neq a} Z_{a,b}(t)$ measures the distance between an instance where the current empirical best arm is indeed the best arm, and the ``closest'' instance where the current empirical best arm is not the true best arm, both reflected through reward observations. So, the larger $Z(t)$ is, the more confident that the empirical best arm is indeed the best arm. The proposition below ensures the $\delta$-PAC guarantee of $\mathsf{CTAS}$.

\begin{proposition}[$\delta$-PAC]\label{prop:beta}
    Let $\delta \in(0,1)$ and $\alpha \geq 1$. There exists a constant $B_{\alpha}$\footnote{$B_\alpha$ satisfies $B_\alpha \geq 2K$ for $\alpha = 1$,  or
$\sum_{t=1}^{\infty} \frac{e^{K+1}}{K^K} \frac{\left(\log ^2\left(B_\alpha t^\alpha\right) \log t\right)^K}{t^\alpha} \leq B_\alpha$ for $\alpha > 1$.}
such that for all $B \geq B_{\alpha}$ the $\mathsf{CTAS}$ algorithm in Algorithm.~\ref{alg:ctas} is $\delta$-PAC, i.e.,
\begin{align*}
    \mathbb{P}_{\boldsymbol{\mu}\times \boldsymbol{c}}\left(\tau_\delta<\infty, \hat{a}_{\tau_\delta} \neq a^*\right) \leq \delta.
\end{align*}
\end{proposition}
The cost bandit setting also encourages the algorithm to stop as early as possible, so the same stopping rules from traditional BAI~\citep{Garovoer16BAI} can be used. A more refined threshold can be found in \citet{kaufmann2021mixture}. However, we will use the threshold in Algorithm \ref{alg:ctas} for the rest of the paper for simplicity. Our Proposition.~\ref{prop:beta} combines Theorem 10 and Proposition 11 from~\citet{Garovoer16BAI}, and the proof will be provided in the appendix.

\textbf{Asymptotic Optimality for $\mathsf{CTAS}$:}
In Theorem.~\ref{thm:exp-opt}, we provide provable cost guarantees for the $\mathsf{CTAS}$ algorithm. Namely, the algorithm asymptotically achieves the lower bound in Theorem~\ref{thm:general-lower-bound} in expectation as the confidence level $\delta$ decreases to $0$. 

\begin{theorem}[Expected Upper Bound]\label{thm:exp-opt}
Let $\delta \in [0,1)$ and $\alpha \in[1, e / 2]$. Using Chernoff's stopping rule with $\beta(t, \delta)=\log (\mathcal{O}(t^\alpha) / \delta)$, the $\mathsf{CTAS}$ algorithm ensures:
$$
\limsup _{\delta \rightarrow 0} \frac{\mathbb{E}_{\boldsymbol{\mu} \times \boldsymbol{c}}\left[J(\tau_\delta)\right]}{\log (1 / \delta)} \leq \alpha T^*(\boldsymbol{\mu}).
$$
\end{theorem}

For optimality, we can simply take $\alpha = 1$ and choose $B \geq 2K$ from Proposition \ref{prop:beta}.
The proof of the theorem along with a weaker almost sure cost upper bound result (Theorem.~\ref{thm:as-upper}) will be provided in the appendix. The major difference of the expected upper bound and the weaker version is the rate of exploration, where Theorem.~\ref{thm:exp-opt} requires $\mathcal{O}(\sqrt{t})$ forced exploration rate while the weaker version suffice with $o(t)$. We first show the empirical proportion of the cost for each arm converges to the optimal proportion (Theorem.~\ref{thm:as-conv}), with the help of forced exploration rate. Then, the Chernoff stopping time ensures our algorithm stops early to guarantee $\delta$-PAC and to minimize the cost. 

\section{Low Complexity Algorithm}

Even though $\mathsf{CTAS}$ achieves asymptotically optimal cost performance, this algorithm suffers from the heavy computation time of computing $\boldsymbol{w}^*$. As shown in Table~\ref{tab:runtime} in Section \ref{sec:experiments}, the $\mathsf{CTAS}$ and $\mathsf{TAS}$ algorithm requires much more time to compute the sampling rule at each time step.
This leads to the desire for a ``model-free'' algorithm that does not require us to compute $\boldsymbol{w}^*$. In this section, we propose a low-complexity algorithm called Chernoff-Overlap ($\mathsf{CO}$) which is summarized in Algorithm~\ref{alg:co}. $\mathsf{CO}$ is based on action elimination. The main idea behind these algorithms is to sample each arm uniformly and then eliminate arms that can be declared sub-optimal with high probability. However, it is easy to see that sampling uniformly would not be a good idea in the case of heterogeneous costs. This requires that the sampling rule take into account the proper ratio of information gained from pulling an arm concerning the cost of that arm.

\textbf{Sampling Rule:}
To gain maximum information on the remaining uncertainty of reward, it is desirable to pull the arm with the largest decrease in ``overlap'' as shown in Fig.~\ref{fig:overlap_ellipses}, which results in the arm with minimum pulls $N_t(a)$. However, we also need to consider the cost of arms and weigh the decrease of overlap with cost. Through analysis of the two-armed Gaussian setting, this leads to our choice of sampling rule which weighs $N_a(t)$ with $\sqrt{c_a}$, called the square-root rule.

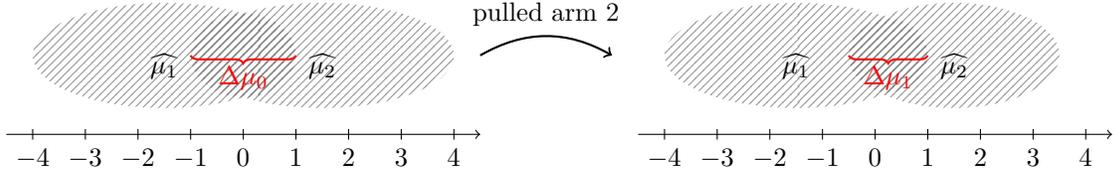
\begin{figure}[t]
    \begin{minipage}{.4\linewidth}
        \begin{tikzpicture}[scale=.7]
        \fill[pattern=north east lines, pattern color=gray!80] (-1.5,0) ellipse (2.5 and 1);
        \draw (-1.5,-.2) node {$\hat{\mu_1}$};
        \fill[pattern=north east lines, pattern color=gray!80] (1.5,0) ellipse (2.5 and 1);
        \fill[pattern=north east lines, pattern color=gray] (0,0) ellipse (1 and .8);
        \draw (1.5,-.2) node {$\hat{\mu_2}$};
        \draw[
            red,
            thick,
            decorate,
            decoration={
                brace,
                amplitude=3pt,
                mirror,
            },
        ]
        (-1, 0) -- (1,0) node[midway, below] {\red{$\Delta \mu_{0}$}};
        \draw[->] (-4.5,-1.5) -- (4.5,-1.5);
                \foreach \x in {-4, -3,-2,-1,0,1,2,3,4}
                    \draw (\x,-1.4) -- (\x,-1.6) node[below] {$\x$};
        \draw[->, bend left=30, thick] (4.5, 0) to node[above] {pulled arm 2} (7,0);
        \fill[pattern=north east lines, pattern color=gray!80] (10.5,0) ellipse (2.5 and 1);
        \draw (10.5,-.2) node {$\hat{\mu_1}$};
        \fill[pattern=north east lines, pattern color=gray!80] (13.5,0) ellipse (2 and 1);
        \fill[pattern=north east lines, pattern color=gray] (12.25,0) ellipse (.75 and .7);
        \draw (13.5,-.2) node {$\hat{\mu_2}$};
        \draw[
            red,
            thick,
            decorate,
            decoration={
                brace,
                amplitude=3pt,
                mirror,
            },
        ]
        (11.5, 0) -- (13,0) node[midway, below] {\red{$\Delta \mu_{1}$}};
        \draw[->] (7.5,-1.5) -- (16.5,-1.5);
                \foreach \x in {8,9,10,11,12,13,14,15,16}
                    \pgfmathsetmacro{\xminusone}{int(\x - 12)}
                    \draw (\x,-1.4) -- (\x,-1.6) node[below] {$\xminusone$};
    \end{tikzpicture}
    \end{minipage}
    \caption{Change in overlap upon pulling arm $2$, where the ellipsoids stand for confidence intervals. Left: wider confidence interval for $\mu_2$. Right: reduced confidence interval upon pulling arm 2.}
    \label{fig:overlap_ellipses}
\end{figure}

\textbf{Stopping Rule:}
Our stopping rule will still rely on the generalized likelihood ratio. For any time $t$, Let $a^*(t)$ be the empirical best arm, i.e., $a^*(t) = \arg\max_{a} \hat{\mu}_a(t)$. The Chernoff statistics we adopt in $\mathsf{CO}$ is instead the pairwise statistic $Z_{a^*(t), a}(t)$.
When it is large, the empirical reward $\hat{\mu}_a(t)$ of arm $a$ is significantly lower than the empirical reward of $a^*(t)$, which gives us high confidence to eliminate this arm. Naturally, we then stop when only one arm remains. The following proposition ensures that by the choice of a proper threshold, $\mathsf{CO}$ is $\delta$-PAC. The proof will be in the appendix.

\begin{proposition}[$\delta$-PAC]\label{prop:beta-co}
    Let $\delta \in(0,1)$ and $\alpha \geq 1$. There exists a large enough constant $B$\footnote{$B$ can be chosen the same as Proposition \ref{prop:beta}.} such that the Chernoff-Overlap algorithm in Algorithm.~\ref{alg:co} is $\delta$-PAC, i.e.,
    \begin{align*}
        \mathbb{P}_{\boldsymbol{\mu}\times \boldsymbol{c}}\left(\tau_\delta<\infty, \hat{a}_{\tau_\delta} \neq a^*\right) \leq \delta.
    \end{align*}
\end{proposition}

\textbf{Cost Upper Bound for Chernoff-Overlap:}
It is difficult to relate an algorithm to the general cost lower bound in Theorem \ref{thm:general-lower-bound} without direct tracking. Therefore, we must resort to relating the cost upper bound of Chernoff-Overlap to the lower bound in cases where there is a closed-form solution. We consider the two-armed Gaussian bandits setting and show that Chernoff-Overlap is asymptotically cost-optimal for this special case, resulting in the following Theorem.

\begin{theorem}\label{thm:chernoff-two-arm}
    Let $\delta\in(0,1)$ and $\alpha \in (1, e/2)$. For any 2-armed Gaussian bandit model with rewards $\{\mu_1, \mu_2\}$ and costs $\{c_1, c_2\}$, under the $\mathsf{CO}$ algorithm in Algorithm \ref{alg:co} we have with probability $1$:
    \begin{align*}
        \limsup_{\delta \to 0} \frac{J(\tau_{\delta})}{\log(1/\delta)} \leq \frac{2\alpha\left(\sqrt{c_1}+\sqrt{c_2}\right)^2}{(\mu_1 - \mu_2)^2}.
    \end{align*}
\end{theorem}
The proof of Theorem.~\ref{thm:chernoff-two-arm} will be delayed to the appendix. The key to the proof is to show that under our sampling rule balanced by $\sqrt{\hat{c_a}(t)}$ for each arm, the empirical cost proportion $\hat{\boldsymbol{w}}(t)$ converges to the optimal proportion $\hat{\boldsymbol{w}}^*$. Then, we can apply a similar argument as the weaker version of Theorem~\ref{thm:exp-opt} to prove the upper bound. Comparing it to Corollary.~\ref{corr:2-arm-gauss}, we show our low-complexity algorithm is optimal in this setting. It is an important observation that in the homogeneous cost case, this algorithm reduces to a racing algorithm. It is well known that racing algorithms cannot be optimal on a general MAB model. However, we will show that it enjoys surprisingly good empirical performance over a wide range of bandit models with multiple arms in the next section. Establishing a provable suboptimality gap is an interesting future research problem. 

\section{Numerical Experiments}
\label{sec:experiments}
As shown before, $\mathsf{CO}$ does not inherit the strong theoretical guarantees of $\mathsf{CTAS}$. However, the main appeal of the algorithm comes from both its simplicity and the much more efficient computation time. As shown in Table \ref{tab:runtime}, $\mathsf{CO}$ takes significantly less time to run while maintaining good performance.

\begin{table}[H]
    \centering
    \begin{tabular}{ccccc}
    \toprule
    & $\mathsf{CO}$ & $\mathsf{CTAS}$ & $\mathsf{TAS}$ & $\mathsf{d}\text{-}\mathsf{LUCB}$  \\
    \midrule
    Gaussian & 85 & 1712 & 2410 & 82 \\
    Bernoulli & 58 & 1995 & 2780 & 60 \\
    Poisson & 96 & 3260 & 4633 & 101 \\
    \bottomrule
    \end{tabular}
    \caption{The process time (seconds) of each of the algorithms over 1000 trajectories for Gaussian, Bernoulli, and Poisson distributed rewards with $\boldsymbol{\mu} = [1.5, 1.0, 0.5]$ and $\boldsymbol{c} = [1, 0.1, 0.01]$.}
    \label{tab:runtime}
\end{table}
\begin{algorithm}[t]

\caption{Chernoff-Overlap Algorithm}\label{alg:co}
    \KwIn{confidence level $\delta$; $\alpha \geq 1$; sufficiently large constant $B$}
    pull each arm $a\in\mathcal{A}$ once as initialization\;
    \For{$t\geq K+1$}{
        \If(\tcp*[f]{Stopping Rule}){$|\mathcal{R}| \leq 1$}{
            \textbf{break};
        }
        eliminate all arms $a$ from $\mathcal{R}$ if $Z_{a^*(t), a}(t)> \log\left(\frac{Bt^{\alpha}}{\delta}\right)$, where $a^*(t)=\arg\max_a\hat{\mu}_a(t)$\;
        pull arm $a_{t} \in \underset{a \in \mathcal{R}}{\operatorname{argmin}}\ \sqrt{c_a} N_a(t)$   \tcp*{Sampling Rule}
    }
    \Return $\hat{a}\in \mathcal{R}$
\end{algorithm}
\textbf{Discussion:} Our square-root sampling rule of $\mathsf{CO}$ comes from reverse-engineering the optimal proportion in the two-armed Gaussian case and then separating the multi-arm problem into pairs of two-armed problems using action elimination. However, an interesting empirical result shown in Fig.~\ref{fig:diff-dist} is how well it generalizes to other reward distributions. This is illustrated in Figure \ref{fig:diff-dist}(b). We see that the change in reward distribution does not drastically impact performance. From this, we can deduce that the cost factor of $\sqrt{c}$ generalizes beyond Gaussian distributions. This is partially because when the shrinkage of confidence interval overlap is small, the exponential distribution family is locally Gaussian, and therefore can be approximated by Gaussian bandits. More evidence for this cost factor is shown in Figure \ref{fig:diff-dist}(a). Here we see that $\mathsf{CO}$ is approximately able to match the optimal proportions of arm pulls. The main distinction is that $\mathsf{CO}$ is more willing to pull the low-cost arm to eliminate it early on. This results in similar performance because the additional pulls are inexpensive relative to the other arms.  Another interesting observation is $\mathsf{CO}$ sometimes performs better than $\mathsf{CTAS}$. This is in part because of the elimination rule in $\mathsf{CO}$. While the same proof as $\mathsf{CTAS}$ can be utilized to show that $\mathsf{CO}$ is $\delta$-PAC, the theory does not utilize the full ``tightness'' of the $\mathsf{CTAS}$ stopping rule. The $\mathsf{CO}$ event of error lives in between the event of error for $\mathsf{CTAS}$ and the event bounded by theory, causing earlier stopping with less confidence. Empirically, we also had to do more exploration by a constant factor of $\sqrt{t}$ due to the added variance from random costs. Lastly, the $\mathsf{TAS}$ family algorithms are very sensitive to good initial starts, meaning that the results are also obfuscated by these extraordinarily long trajectories due to insufficient exploration.
\section{Conclusion}
In this work, we introduced a new MAB problem: Cost-Aware Best Arm Identification. We provided a new lower bound and an asymptotically optimal cost-adapted BAI algorithm. Finally, we introduced a low-complexity algorithm with promising empirical results.  As a future direction, it may be interesting to explore how this algorithm can be adapted to the regret setting in either the cost adapted setting \citep{RegretCost}, or as an ETC algorithm for carefully chosen costs.

\subsubsection*{Acknowledgments}\label{sec:ack}
This work is supported in part by NSF under grants 2112471, 2134081, 2207548, 2240981, and 2331780.

\begin{figure}[t]
\begin{minipage}{.49\textwidth}
\begin{minipage}{0.49\textwidth}
\includegraphics[width=\textwidth]{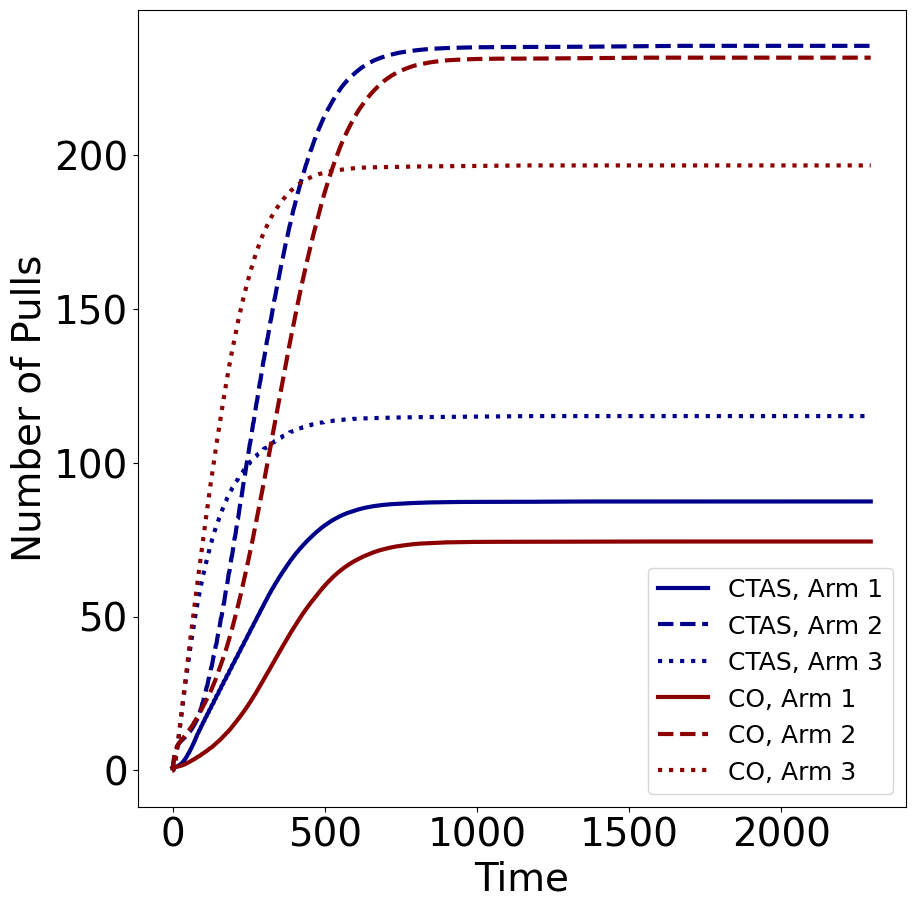}
\end{minipage}
\begin{minipage}{0.49\textwidth}
\includegraphics[width=\textwidth]{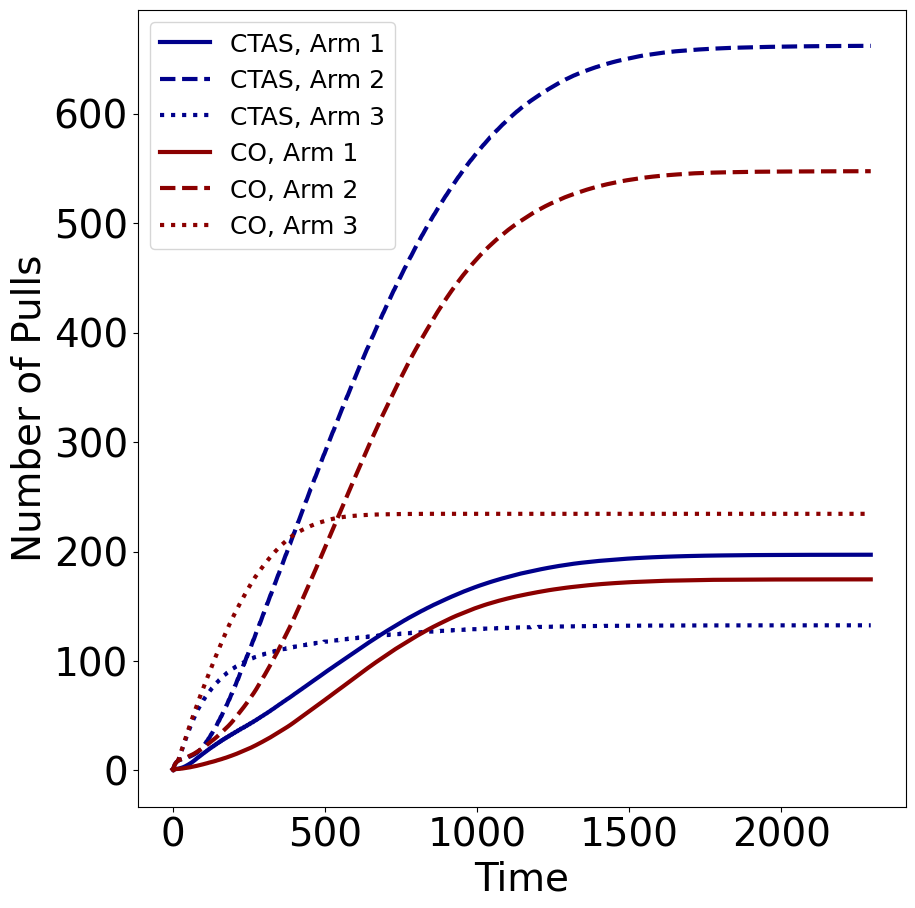}
\end{minipage}
\begin{center}
    (a) Arm pull trajectories
\end{center}
\end{minipage}
\begin{minipage}{.49\textwidth}
\begin{minipage}{0.49\textwidth}
\includegraphics[width=\textwidth]{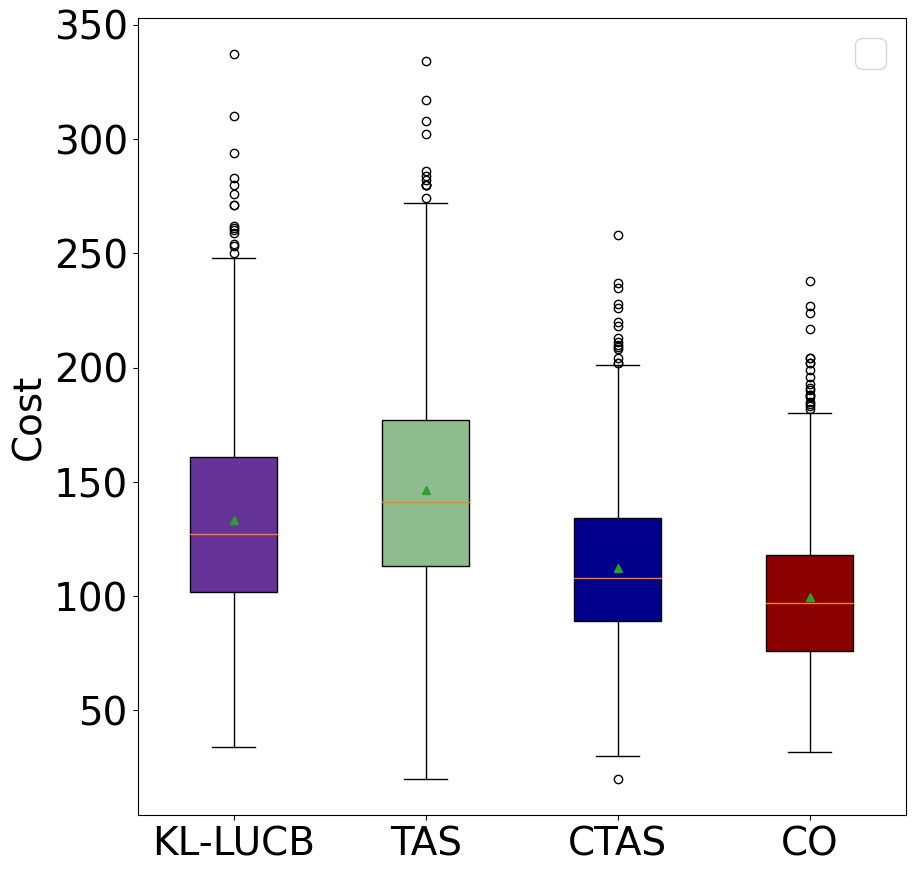}
\end{minipage}
\begin{minipage}{0.49\textwidth}
\includegraphics[width=\textwidth]{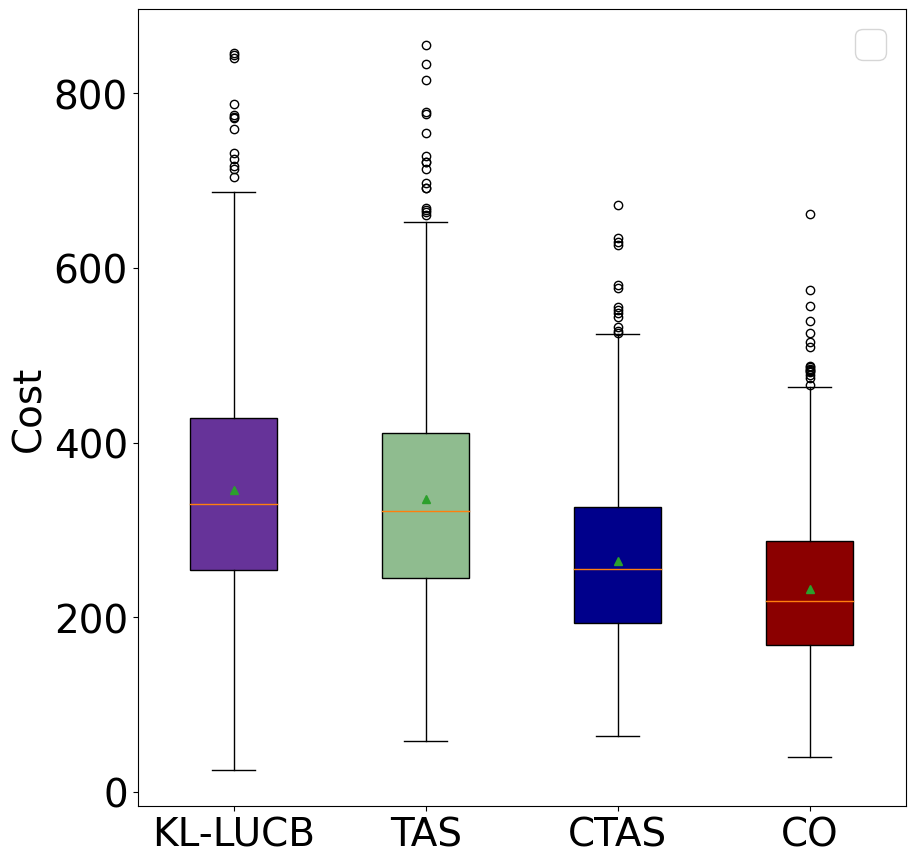}
\end{minipage}
\begin{center}
    (b) Total costs.
\end{center}
\end{minipage}
\caption{Results averaged over $1000$ trajectories with fixed confidence level $\delta = 10^{-6}$. In (a), we have the average number of arm pulls at each time $t$. In (b) we have the statistics regarding total cost for these trajectories. This figure was generated with $\boldsymbol{\mu}_1 = [1.5, 1, .5]$ and $\boldsymbol{\mu}_2 = [.9, .6, .3]$ with $\boldsymbol{c} = [1, .1, .01]$, where $\boldsymbol{\mu}_1$ and $\boldsymbol{\mu}_2$ follow a Bernoulli and Poisson distribution respectively.}
\label{fig:diff-dist}
\end{figure}

\newpage

\bibliography{ref_RLC}
\bibliographystyle{rlc}

\newpage 

\appendix

\section{Related Work}
In this section, we review the other classic problem formulations of the $\mathsf{MAB}$ problem. Here we provide a detailed discussion of less directly related works
\subsection{Best Arm Identification}
Traditional best arm identification (BAI) has been studied for decades \citep{bechhofer1958sequential}. In this setting, there are two possible objectives. First, we have the fixed confidence setting, where the agent aims to minimize the sample time (or the number of samples) while ensuring the best arm is identified with probability greater than $1 - \delta$. Second, there is the fixed budget setting, where the agent has a fixed number of samples and aims to minimize the probability of error. 
\subsubsection{Fixed Confidence}
In \citet{kaufmann2016complexity}, they introduce a non-asymptotic lower bound for BAI. Subsequently, they propose the Track and Stop algorithm ($\mathsf{TAS}$) that achieves this lower bound asymptotically. However, this algorithm requires the computation of the optimal ``proportion'' of arm pulls, $w^*$, at each iteration and requires forced exploration at the order of $\sqrt{\tau}$. The $\mathsf{TAS}$ algorithm has since been extended to a variety of other settings \citep{pmlr-v201-jourdan23a, garivier2021nonasymptotic, kato2021role}. In an attempt to ease computation time, \citet{menard2019gradient} implements a gradient ascent algorithm to get around having to compute the optimal proportions at each iteration.
\par
In \citet{degenne2019non}, they introduce an optimal algorithm in the non-asymptotic regime. This algorithm formulates the lower bound as a game between two zero-regret algorithms and pulls arms based on the outputted proportions of one player. However, the algorithm is even more computationally expensive than vanilla $\mathsf{TAS}$. For sub-gaussian rewards, \citet{barrier2022non} proposes an optimal algorithm with non-asymptotic optimal upper bound. Here, they also address forced exploration by instead computing the optimal proportions with high probability upper bounds on the empirical means at each iteration. However, the results do not hold for general exponential families.
\par
Prior to these ``model-based'' methods, there were many ``confidence-based'' algorithms \citep{2012kl, Kaufmann13kl, UGapE} focused on constructing high-probability confidence intervals. The stopping time is then when one confidence interval was disjoint from and greater than all the rest. While these algorithms are more computationally efficient than the aforementioned ``model-based'' algorithms, they do not perform nearly as well. These algorithms ideas were heavily inspired by the regret minimization setting and the heralded UCB algorithm \citep{LAI19854, Capp__2013, UCB-EXP}.
\subsubsection{Fixed Budget}
The budgeted $\mathsf{MAB}$ formulation was introduced in \citet{BwK} as a knapsack problem and allows the reformulation of the $\mathsf{MAB}$ problem to an assortment of constraints. As alluded to in \citet{RegretCost}, this framework is not generalizable to problem-dependent cost constraints as needed in our formulation. 
\par
From this, successive elimination algorithms emerged \citep{audibert2010best}. In these algorithms, the agent pulls every arm at each round then sees if any arms can be eliminated with high probability. The agent then stops when only one arm is remaining. This approach was proven optimal in \citet{carpentier2016tight} for compactly supported distributions and provides inspiration for our low-complexity algorithm.
\par
Cost has been introduced into the fixed budget case as a natural extension of the budget. Specifically, non-deterministic cost has been explored in the fixed budget setting for both discrete \citep{ding2013multi} and continuous cost \citep{xia2016budgeted}. In both cases, the density functions of the costs are compactly supported. Fixed budget BAI can be thought of as the dual problem to our problem setting. Instead of trying to maximize the confidence subject to a hard cost constraint, we try to minimize the cost subject to a hard confidence constraint.

\subsubsection{Best Arm Identification with Safety Constraints}
A somewhat similar formulation to the one in this work is the addition of safety constraints to the BAI problem. In \citet{wang2022best}, they propose a similar clinical trial example. However, each drug is associated with a dosage and the dosage has an associated safety level. They then attempt to identify the best drug and dosage level for some fixed confidence and safety level. On a similar note, \citet{hou2022almost} aims at identifying the best arm subject to a constraint on the variance of the best arm. In some way, our formulation can be viewed as BAI with soft constraints, where the agent is discouraged from certain actions but not forbidden. Additionally, in our formulation we allow for the best arm to be high cost unlike in \citet{hou2022almost}, where they require that the chosen arm satisfy the variance constraint.
\subsection{Regret Minimization}
Regret minimization is concerned with minimizing the number of pulls of suboptimal arms or equivalently maximizing the number of times the arm with the highest reward is pulled. Confidence interval based algorithms have proven effective for this setting \citep{LAI19854, UCB-EXP}. This led to fruitful work in finding tighter confidence intervals utilizing Kullback-Liebler divergence \citep{Capp__2013}. 
\subsubsection{Explore Then Commit}
Related to BAI, there is a category of regret minimization algorithms that follow an explore-then-commit approach \citep{Garovoer16ETC}. In these algorithms, the agent first attempts to identify the best arm as in BAI, then the agent commits to only pulling the identified best arm. However, it was shown in \citet{degenne2019bridging} that any ETC algorithm that first utilizes an optimal BAI algorithm will not achieve the optimal regret bounds from \citet{LAI19854, 1996Reg}. A further refined trade-off between best arm identification and regret minimization is studied in~\citet{zhang2024robai}.
\subsubsection{Regret Minimization With Cost}
The addition of cost has recently been explored for regret minimization. In \citet{RegretCost}, they introduce the cost as a separate objective, showing that no algorithm can optimally reduce the cost regret and reward regret. Similar to our work, they introduce cost-adapted versions of traditional regret-minimization algorithms. In \citet{CostExplore}, they introduce cost into the exploration-phase of an ETC algorithm. Here, the agent incurs a cost for each exploration step and is then tasked with minimizing a linear combination of the cost and regret. This is to circumvent the suboptimality of optimal BAI algorithms in regret minimization shown in \citet{degenne2019bridging} by introducing regret into the objective.

\section{Proofs for the Lower Bound Theorem and Related Corollaries}\label{app:lower-bound}

\subsection{Proof of the General Lower Bound Theorem.~\ref{thm:general-lower-bound}}
This section presents a formal proof of Theorem~\ref{thm:general-lower-bound}. Recall that our goal is to lower bound the expected cumulative cost $\mathbb{E}[J(\tau_\delta)]$, where $\tau_\delta$ is a stopping time depending on both the randomness of cost samples and the best arm identification algorithm. The following cost decomposition lemma is useful which is an analog to the classic regret decomposition lemma~\cite[Lemma 4.5]{lattimore2020bandit}.

\begin{lemma}[Cost Decomposition Lemma]\label{lemma:cost-decomp}
For any $K$-armed stochastic bandit environment, for any algorithm with stopping time $\tau$, the cumulative cost incurred by the algorithm satisfies:
\begin{equation}\label{eq:cost-decomp}
\mathbb{E}\left[J(\tau)\right] = \sum_{a = 1}^{K} c_a \cdot \mathbb{E}\left[N_a(\tau)\right].
\end{equation}
\end{lemma}

The proof of Lemma.~\ref{lemma:cost-decomp} closely follows the proof of the regret decomposition lemma. The only difference is that now we have to deal with a stopping time $\tau$ instead of a fixed time horizon, which can be achieved through applying the tower property carefully. With the help of the cost decomposition lemma, we are ready to prove Theorem.~\ref{thm:general-lower-bound}.

\begin{proof}[\textbf{Proof of Theorem \ref{thm:general-lower-bound}}]
Let $\delta \in(0,1)$ be the confidence level, let $\boldsymbol{\mu} \in \mathcal{S}$ be any bandit instance that we study, and consider a $\delta$-PAC algorithm. For any time index $t \geq 1$, denote by $N_a(t)$ the (random) number of draws of arm $a$ up to time $t$. 

Similar to the proof of Theorem 1 in~\cite{Garovoer16BAI}, we first invoke the 'transportation' lemma~\cite[Lemma 1]{kaufmann2016complexity}, \cite[Theorem 1]{Garovoer16BAI} to build a relationship between the expected number of draws $N_a(t)$ and the Kullback-Leibler divergence from bandit instance $\boldsymbol{\mu}$ to another bandit model $\boldsymbol{\lambda}$ which has a different best arm: for any $\boldsymbol{\lambda}\in \mathcal{S}$ such that $a^*(\boldsymbol{\lambda}) \neq a^*(\boldsymbol{\mu})$, we have:
\begin{align*}
    \sum_{a=1}^K d\left(\mu_a, \lambda_a\right) \mathbb{E}_{\boldsymbol{\mu}}\left[N_a\left(\tau_\delta\right)\right] \geq \sup_{\mathcal{E}\in \mathcal{F}_{\tau_\delta}}\operatorname{d}\left(\mathbb{P}_{\boldsymbol{\mu}}(\mathcal{E}), \mathbb{P}_{\boldsymbol{\lambda}}(\mathcal{E})\right).
\end{align*}
Selecting the event $\mathcal{E} = \{\hat{a} = 1\}$, we notice that our considered algorithm is $\delta$-PAC, so we require $\mathbb{P}_{\boldsymbol{\mu}}(\mathcal{E})\geq 1- \delta$ and $\mathbb{P}_{\boldsymbol{\lambda}}(\mathcal{E})\leq 1-\delta$, since the optimal arm for instance $\boldsymbol{\mu}$ is arm $1$ but the optimal arm for instance $\boldsymbol{\lambda}$ is not arm $1$. Together with the symmetry and monotonicity of the $\operatorname{d}$ function, we have:
\begin{align*}
    \operatorname{d}\left(\delta, 1-\delta\right) 
    \leq 
    \sum_{a=1}^K d\left(\mu_a, \lambda_a\right) \mathbb{E}_{\boldsymbol{\mu}}\left[N_a\left(\tau_\delta\right)\right].
\end{align*}
Then, we can choose the instance $\boldsymbol{\lambda}\in \operatorname{Alt}(\boldsymbol{\mu})$ which has a different best arm other than instance $\boldsymbol{\mu}$, to achieve the inf of the right hand side. With a little manipulation, we have:
\begin{align*}
    \operatorname{d}(\delta, 1-\delta) 
    \leq  \inf _{\lambda \in \operatorname{Alt}(\boldsymbol{\mu})} \sum_{a=1}^K d\left(\mu_a, \lambda_a\right) \mathbb{E}_{\boldsymbol{\mu}}\left[N_a\left(\tau_\delta\right)\right]
    =\inf _{\lambda \in \operatorname{Alt}(\boldsymbol{\mu})} \mathbb{E}_{\boldsymbol{\mu}}\left[J(\tau_{\delta})\right]\left(\sum_{a=1}^K \frac{\mathbb{E}_{\boldsymbol{\mu}}\left[N_a(\tau_\delta)\right]}{\mathbb{E}_{\boldsymbol{\mu}}\left[J(\tau_{\delta})\right]} d\left(\mu_a, \lambda_a\right)\right) .
\end{align*}
Now we use the cost decomposition lemma~\ref{lemma:cost-decomp} to decompose the cumulative cost in the denominator, and then we have:
\begin{align*}
    \operatorname{d}(\delta, 1-\delta)  
    =& \mathbb{E}_{\boldsymbol{\mu}}\left[J(\tau_{\delta})\right] \inf _{\lambda \in \operatorname{Alt}(\boldsymbol{\mu})}\left(\sum_{a=1}^K \frac{1}{c_a} \frac{ c_a \mathbb{E}_{\boldsymbol{\mu}}\left[ N_a(\tau_{\delta})\right]}{ \mathbb{E}_{\boldsymbol{\mu}}\left[J(\tau_\delta)\right]} d\left(\mu_a, \lambda_a\right)\right) \\
    =& 
    \mathbb{E}_{\boldsymbol{\mu}}\left[J(\tau_{\delta})\right] \inf _{\lambda \in \operatorname{Alt}(\boldsymbol{\mu})}\left(\sum_{a=1}^K \frac{1}{c_a} \frac{ c_a \mathbb{E}_{\boldsymbol{\mu}}\left[ N_a(\tau_{\delta})\right]}{ \sum_{a=1}^K c_a \mathbb{E}_{\boldsymbol{\mu}}\left[N_a(\tau_\delta)\right]} d\left(\mu_a, \lambda_a\right)\right).
\end{align*}
We let $w_a = \frac{c_a \mathbb{E}_{\boldsymbol{\mu}}\left[ N_a(\tau_{\delta})\right]}{\sum_{a=1}^K c_a \mathbb{E}_{\boldsymbol{\mu}}\left[ N_a(\tau_{\delta})\right]}$ be the proportion of cost incurred by arm $a$ which sums up to 1. Then, we can further upper bound the above expression by taking the sup over $\boldsymbol{w}$ as:
\begin{align*}
    \operatorname{d}(\delta, 1-\delta)  
    \leq \mathbb{E}_{\boldsymbol{\mu}}\left[J(\tau_\delta)\right] \sup _{\boldsymbol{w} \in \Sigma_K} \inf _{\operatorname{Alt}(\boldsymbol{\mu})}\left(\sum_{a=1}^K \frac{w_a}{c_a} d\left(\mu_a, \lambda_a\right)\right) 
    \leq \mathbb{E}_{\boldsymbol{\mu}}\left[J(\tau_\delta)\right] T^*(\boldsymbol{\mu})^{-1}.
\end{align*}
Thus, we can finally derive the lower bound on the cumulative cost as desired:
\begin{align*}
    \mathbb{E}_{\boldsymbol{\mu}}\left[J(\tau_\delta)\right] \geq T^*(\boldsymbol{\mu})d(\delta, 1 - \delta).
\end{align*}
\end{proof}

\subsection{Proof of Lemma \ref{lemma:cost-decomp}}
\begin{proof}[\textbf{Proof of Lemma.~\ref{lemma:cost-decomp}}]
    For any stopping time $\tau$, we apply the tower property to condition on the stopping time $\tau$ as follows:
\begin{align*}
    \mathbb{E}[J(\tau)] 
    = \mathbb{E}\left[ \sum_{t=1}^\tau C_t \right] 
    = \mathbb{E} \left[\sum_{a=1}^K \sum_{t=1}^\tau C_t \mathbf{1}_{A_t = a} \right] 
    = \sum_{a=1}^K \mathbb{E}\left[ \mathbb{E}\left[ \left.\sum_{t=1}^\tau C_t \mathbf{1}_{A_t=a} \right|\tau\right] \right].
\end{align*}
On event $\mathbf{1}_{A_t = a}$, $C_t$ is a random variable with expectation $c_a$ and independent of the stopping time $\tau$, so we have:
\begin{align*}
    \mathbb{E}\left[ \left.\sum_{t=1}^\tau C_t \mathbf{1}_{A_t=a} \right|\tau\right] 
    = c_a \mathbb{E}\left[ \left.\sum_{t=1}^\tau \mathbf{1}_{A_t=a} \right|\tau\right] 
    = c_a \mathbb{E}\left[ \left. N_a (\tau) \right|\tau \right].
\end{align*}
Therefore, we can put everything back and reuse the tower property as:
\begin{align*}
    \sum_{a=1}^K \mathbb{E}\left[ \mathbb{E}\left[ \left.\sum_{t=1}^\tau C_t \mathbf{1}_{A_t=a} \right|\tau\right] \right] 
    = \sum_{a=1}^K c_a \mathbb{E}\left[ \mathbb{E}\left[ \left.N_a (\tau) \right|\tau\right] \right] = \sum_{a=1}^K c_a \mathbb{E}[N_a(\tau)].
\end{align*}

\end{proof}

\subsection{Jensen-Shannon Divergence Form of Lower Bound}

Before proving the lower bound theorems in special Gaussian bandit models, we first provide an alternative characterization of the lower bound together with the instance-dependent constants $T^*(\boldsymbol{\mu})$ and $\boldsymbol{w}^*(\boldsymbol{\mu})$. Similar to~\cite{Garovoer16BAI}, for two expected rewards $\mu_1$ and $\mu_2$, we introduce the identically parameterized Jensen-Shannon divergence $I_{\alpha}(\mu_1, \mu_2)$ as follows:
\begin{equation}\label{eq:jens-shan}
    I_{\alpha}(\mu_1, \mu_2) \coloneqq \alpha d(\mu_1, \alpha \mu_1 + (1 - \alpha)\mu_2) + (1 - \alpha) d(\mu_2, \alpha \mu_1 + (1 - \alpha)\mu_2).
\end{equation}
The following Proposition characterizes the instance-dependent minimax problem from Theorem.~\ref{thm:general-lower-bound} using this Jensen-Shannon divergence:
\begin{proposition}\label{prop:jens-shan}
    For every $w \in \Sigma_K$,
    $$
    \inf _{\boldsymbol{\lambda} \in \operatorname{Alt}(\boldsymbol{\mu})}\left(\sum_{a=1}^K \frac{w_a}{c_a} d\left(\mu_a, \lambda_a\right)\right)=\min _{a \neq 1}\left(\frac{w_1}{c_1}+\frac{w_a}{c_a}\right) I_{\frac{w_1/c_1}{w_1/c_1+w_a/c_a}}\left(\mu_1, \mu_a\right) .
    $$
    It follows that
    $$
    \begin{aligned}
    T^*(\boldsymbol{\mu})^{-1} & =\sup _{w \in \Sigma_K} \min _{a \neq 1}\left(\frac{w_1}{c_1}+\frac{w_a}{c_a}\right) I_{\frac{w_1/c_1}{w_1/c_1+w_a/c_a}}\left(\mu_1, \mu_a\right), \\
    w^*(\boldsymbol{\mu}) & =\underset{w \in \Sigma_K}{\operatorname{argmax}} \min _{a \neq 1}\left(\frac{w_1}{c_1}+\frac{w_a}{c_a}\right) I_{\frac{w_1/c_1}{w_1/c_1+w_a/c_a}}\left(\mu_1, \mu_a\right) .
    \end{aligned}
    $$
\end{proposition}
\begin{proof}[\textbf{Proof of Proposition.~\ref{prop:jens-shan}}]
Let $\boldsymbol{\mu}$ such that $\mu_1>\mu_2 \geq \cdots \geq \mu_K$. Using the fact that
$$
\operatorname{Alt}(\boldsymbol{\mu})=\bigcup_{a \neq 1}\left\{\boldsymbol{\lambda} \in \mathcal{S}: \lambda_a>\lambda_1\right\},
$$
one has
$$
\begin{aligned}
T^*(\boldsymbol{\mu})^{-1} & =\sup _{w \in \Sigma_K} \min _{a \neq 1} \inf _{\boldsymbol{\lambda} \in \mathcal{S}: \lambda_a>\lambda_1} \sum_{a=1}^K \frac{w_a}{c_a} d\left(\mu_a, \lambda_a\right) \\
& =\sup _{w \in \Sigma_K} \min _{a \neq 1} \inf _{\boldsymbol{\lambda}\in \mathcal{S}: \lambda_a \geq \lambda_1}\left[\frac{w_1}{c_1} d\left(\mu_1, \lambda_1\right)+\frac{w_a}{c_a} d\left(\mu_a, \lambda_a\right)\right] ,
\end{aligned}
$$
where the second equality is true because the inner inf is achieved when $\lambda_{a'} = \mu_{a'}$ for all $a'\notin \{1,a\}$, and thus $d(\lambda_{a'},\mu_{a'}) =0$. Let $f(\lambda_1,\lambda_a)$ to be the function inside the bracket as follows:
$$
f\left(\lambda_1, \lambda_a\right)=\frac{w_1}{c_1} d\left(\mu_1, \lambda_1\right)+\frac{w_a}{c_a} d\left(\mu_a, \lambda_a\right).
$$
Then optimizing $f(\lambda_1,\lambda_a)$ under the constraint $\lambda_a \geq \lambda_1$ is a convex optimization problem that can be solved analytically. The minimum is obtained for
$$
\lambda_1=\lambda_a=\frac{w_1/c_1}{w_1/c_1+w_a/c_a} \mu_1+\frac{w_a/c_a}{w_1/c_1+w_a/c_a} \mu_a
$$
and value of $T^*(\boldsymbol{\mu})^{-1}$ can be rewritten $\left(\frac{w_1}{c_1}+\frac{w_a}{c_1}\right) I_{\frac{w_1/c_1}{w_1/c_1+w_a/c_a}}\left(\mu_1, \mu_a\right)$, using the function $I_\alpha$ defined in (\ref{eq:jens-shan}).
\end{proof}

\subsection{Proof of Corollary.~\ref{corr:2-arm-gauss}}

\begin{proof}[\textbf{Proof for Corollary.~\ref{corr:2-arm-gauss}}]
    We start with the expression of $T^*(\boldsymbol{\mu})^{-1}$. Recall its expression in two armed bandit models with $\mu_1 > \mu_2$:
\begin{align*}
    T^*(\boldsymbol{\mu})^{-1} = \sup_{w_1 + w_2 = 1}\inf_{\lambda_2 > \lambda_1 } \frac{w_1}{c_1}d(\mu_1,\lambda_1) + \frac{w_2}{c_2}d(\mu_2,\lambda_2).
\end{align*}
Recall that under unit-variance Gaussian bandit models, the d-divergence of two distributions with mean $\mu$ and $\lambda$ is simply $\frac{1}{2}(\mu - \lambda)^2$, so we have:
\begin{align*}
    T^*(\boldsymbol{\mu})^{-1} = \sup_{w_1 \in[0,1]}\inf_{\lambda_2 > \lambda_1 } \frac{w_1 \left(\mu_1 - \lambda_1 \right)^2}{2c_1} + \frac{(1-w_1)\left( \mu_2 - \lambda_2  \right)^2 }{2c_2}.
\end{align*}
Fix $w_1$ and solve the inner optimization problem, we have:
\begin{align*}
    \lambda_1 = \lambda_2 = \frac{c_2 w_1 }{c_2w_1 + c_1(1-w_1)}\mu_1 + \frac{c_1(1-w_1)}{c_2w_1 + c_1(1-w_1)}\mu_2.
\end{align*}
Substitute the expression of $\lambda_1$ and $\lambda_2$ into the outer optimization, we have:
\begin{align*}
    T^*(\boldsymbol{\mu})^{-1} 
    =& \sup_{w_1 \in[0,1]} \frac{w_1(1-w_1)^2 c_1(\mu_1 - \mu_2)^2}{2\left(c_2w_1 + c_1(1-w_1) \right)^2} + \frac{w_1^2(1-w_1) c_2(\mu_1 - \mu_2)^2}{2\left(c_2w_1 + c_1(1-w_1) \right)^2}\\
    =& \frac{(\mu_1 - \mu_2)^2}{2}\sup_{w_1 \in[0,1]}\frac{w_1(1-w_1)}{c_2w_1 + c_1(1-w_1) }.
\end{align*}
Then, solving this maximization problem gives us:
\begin{align*}
    w_1 = \frac{\sqrt{c_1}}{\sqrt{c_1}+\sqrt{c_2}}.
\end{align*}
Substitute the above expression into $T^*(\boldsymbol{\mu})^{-1}$, we can easily get:
\begin{align*}
    T^*(\boldsymbol{\mu})^{-1}  = \frac{(\mu_1 - \mu_2)^2}{2\left( \sqrt{c_1} + \sqrt{c_2} \right)^2}.
\end{align*}
Based on this expression of $T^*(\boldsymbol{\mu})^{-1}$ and Theorem.~\ref{thm:general-lower-bound}, the lower bound of cumulative cost for Gaussian two-armed bandit model with unit variance is simply:
\begin{align*}
    \mathbb{E}\left[J(\tau_{\delta})\right]
    \geq T^*(\boldsymbol{\mu}) \operatorname{d}(\delta, 1-\delta) 
    = \frac{2\left(\sqrt{c_1}+\sqrt{c_2}\right)^2 \operatorname{d}(\delta,1-\delta)}{(\mu_1 - \mu_2)^2}.
\end{align*}
\end{proof}

\subsection{Lower Bound on Another Simplified MAB Model}
Similar to the 2-armed Gaussian case, we can get an explicit lower bound on a slightly more general MAB model. Namely, we can handle an arbitrary amount of suboptimal arms provided that they share reward and cost distribution.

\begin{corollary}\label{corr:K-arm-gauss}
    Let $\delta\in(0,1)$. For any $\delta$-PAC algorithm and any $K$-armed unit-variance Gaussian bandit model with mean $\mu_1 > \mu_2 = \cdots = \mu_K$ and $c_1 \neq c_2 = \cdots = c_K$, we have for any $a\in\mathcal{A}:$
    \begin{align*}
        \mathbb{E}_{\boldsymbol{\mu}\times \boldsymbol{c}}\left[J(\tau_{\delta})\right] \geq \frac{2 \left( c_1 + c_a + \frac{K\sqrt{c_1c_a}}{\sqrt{K-1}}\right) }{(\mu_1 - \mu_a)^2} d(\delta,1-\delta).
    \end{align*}
\end{corollary}

\begin{proof}[\textbf{Proof of Corollary.~\ref{corr:K-arm-gauss}}]
    Let $\Delta = |\mu_1 - \mu_a|$. By Proposition.~\ref{prop:jens-shan}, and together with the definition of Jensen-Shannon divergence in Eq.~\eqref{eq:jens-shan} for Gaussian bandits, we can write $T^*(\boldsymbol{\mu})$ in closed form as:
\begin{align*}
    T^*(\boldsymbol{\mu})^{-1} =& \max_{w: w_1 + (K-1)w_a = 1} \left( \frac{w_1}{c_1} + \frac{w_a}{c_a} \right) \left(\alpha \frac{(1-\alpha)^2(\mu_1 - \mu_a)^2}{2} + (1 - \alpha) \frac{\alpha^2(\mu_1-\mu_a)^2}{2} \right)\\
    =& \max_{w: w_1 + (K-1)w_a = 1} \left( \frac{w_1}{c_1} + \frac{w_a}{c_a} \right) \alpha(1-\alpha)\frac{\Delta^2}{2},
\end{align*}
where $\alpha = \frac{w_1/c_1}{w_1/c_1 + w_a/ c_a}$ and $\Delta = \mu_1 - \mu_a$. So further, we have:
\begin{align*}
    T^*(\mu)^{-1} 
    =& \max_{w: w_1 + (K-1)w_a = 1} \frac{\frac{w_1w_a}{c_1c_a}}{\frac{w_1}{c_1}+\frac{w_a}{c_a}}\frac{\Delta^2}{2}
\end{align*}
substitute $w_1$ with $1-(K-1)w_a$ and take derivatives to maximize it, we can get:
\begin{align*}
    w_1^* =& \frac{\sqrt{c_1}}{\sqrt{c_1}+\sqrt{(K-1)c_a}} =  \frac{\sqrt{(K-1)c_1}}{\sqrt{(K-1)c_1}+(K-1)\sqrt{c_a}}\\
    w_a^* =& \frac{\sqrt{c_a}}{\sqrt{K-1}(\sqrt{c_1}+\sqrt{c_a(K-1)})} = \frac{\sqrt{c_a}}{\sqrt{(K-1)c_1}+(K-1)\sqrt{c_a}}.
\end{align*}

So substitute the expression in, we have:
\begin{align*}
    T^*(\mu)^{-1} = \frac{\sqrt{K-1}}{\left(\sqrt{(K-1)c_1} + \sqrt{c_a}\right)\left( \sqrt{(K-1)c_a} + \sqrt{c_1}\right)} \frac{\Delta^2}{2} = \frac{\Delta^2}{2 \left( c_1 + c_a + \frac{K\sqrt{c_1c_a}}{\sqrt{K-1}}\right)}.
\end{align*}
Thus the lower bound is
\begin{align*}
    \mathbb{E}\left[ \sum_a c_a N_{\tau}(a) \right] \geq \frac{2 \operatorname{d}(\delta,1-\delta)}{\Delta^2}\left( c_1 + c_a + \frac{K\sqrt{c_1c_a}}{\sqrt{K-1}}\right) .
\end{align*}
\end{proof}

\section{Computing the Optimal Proportion $\boldsymbol{w}^*$}

\subsection{Characterization of Optimal Proportion $\boldsymbol{w}^*$}

As in \cite{Garovoer16BAI}, for every $a \in\{2, \ldots K\}$, we define the $g_a(x): \mathbb{R} \rightarrow \mathbb{R}$ function which resembles the expression of instance dependent constant $T^*(\boldsymbol{\mu})$ as follows:
$$
g_a(x)=(1+x) I_{\frac{1}{1+x}}\left(\mu_1, \mu_a\right) .
$$
The function $g_a$ is a strictly increasing one-to-one mapping from $\left[0,+\infty\right)$ onto $\left[0, d\left(\mu_1, \mu_a\right)\right)$ \cite[Theorem 5]{Garovoer16BAI}. Next, we define its inverse function $x_a(y): \mathbb{R} \rightarrow \mathbb{R}$ as follows:
\begin{align*}
    x_a(y)=g_a^{-1}(y).
\end{align*}
So $x_a(y)$ is also a strictly increasing one-to-one mapping from $[0, d\left(\mu_1, \mu_a\right))$ to $[0,+\infty)$.
Using this notation, we can simplify the notation for $T^*(\boldsymbol{\mu})$ according to Proposition~\ref{prop:jens-shan} as follows: 
\begin{align*}
   T^*(\mu)^{-1} &= \underset{w \in \Sigma_K}{\operatorname{argmax}} \min_{a \neq 1} \left(\frac{w_1}{c_1} + \frac{w_a}{c_a}\right) I_{\frac{w_1/c_1}{w_1/c_1 + w_a/ c_a}}(\mu_1, \mu_a) \\
   &= \underset{w \in \Sigma_K}{\operatorname{argmax}} \left[\frac{w_1}{c_1} \min_{a \neq 1} g_a\left(\frac{w_a/c_a}{w_1/c_1}\right)\right]
\end{align*}
It is clear that the solution to the max-min problem resides in the critical point where all $g_a(\frac{w_a/c_a}{w_1/c_1})$ are equal to each other. Then, the problem of computing the optimal proportion $\boldsymbol{w}^*$ becomes solving the balancing point for $g_a(x)$ functions. Let $x_1^*$ being a constant $1$. 
In Lemma~\ref{lemma:g-eq}, we show that this balancing point of the $g_a(x)$ function is exactly the solution $\boldsymbol{x}^* = (x_1^*, \ldots, x_a^*)$ to the following maximization problem:
\begin{align}\label{eq:g-func}
    \underset{(x_1, \ldots, x_K)}{\operatorname{argmax}} \left[\frac{\min_{a \neq 1} g_a(x_a)}{c_1 + c_2 x_2 + \cdots + c_K x_K}\right].
\end{align}

\begin{lemma}\label{lemma:g-eq}
    For any two different sub-optimal arm $a$ and $b$, we have that  $g_a(x_a^*) = g_b(x_b^*)$, which means:
    \begin{align*}
        \underset{w \in \Sigma_K}{\operatorname{argmax}} \left[\frac{w_1}{c_1} \min_{a \neq 1} g_a\left(\frac{w_a/c_a}{w_1/c_1}\right)\right] = 
           \underset{(x_1, \ldots, x_K)}{\operatorname{argmax}} \left[\frac{\min_{a \neq 1} g_a(x_a)}{c_1 + c_2 x_2 + \cdots + c_K x_K}\right].
    \end{align*}
\end{lemma}
\noindent Let $y^* = g_a(x_a^*)$ for all arm $a\in\mathcal{A}$ and notice that $x_a^* = x_a(y^*)$, so we can rewrite the optimization problem.~\eqref{eq:g-func} as follows:
\begin{align*}
    y^* = \underset{y}{\operatorname{argmax}}\ G(y) \coloneqq \frac{y}{c_1 + c_2 x_2(y) + \cdots + c_K x_K(y)}.
\end{align*}
To solve $y^*$, one simply can let the derivative of $G(y)$ equal $0$. After obtaining $y^*$, it is simple to obtain $\boldsymbol{w}^*$ through the inverse mapping $x_a(y)$. Then, we summarize the characterization of $\boldsymbol{w}^*$ in the following Theorem.~\ref{thm:opt-prop}.

\begin{theorem}\label{thm:opt-prop}
The optimal proportion $\boldsymbol{w}^* = (w_1^*, \cdots, w_K^*)$ can be computed through:
\begin{align*}
    w_a^* = \frac{c_a x_a(y^*)}{c_1 + c_2 x_2(y^*) + \cdots + c_K x_K(y^*)},
\end{align*}
where $y^*$ is the unique solution to $F_{\boldsymbol{\mu}, \boldsymbol{c}}(y) = 1$ for 
\begin{align*}
    F_{\boldsymbol{\mu}, \boldsymbol{c}}: y \mapsto \sum_{a = 2}^{K} \frac{c_a \cdot d(\mu_1, \frac{\mu_1 + x_a(y)\mu_a}{1 + x_a(y)})}{c_1 \cdot d(\mu_a, \frac{\mu_1 + x_a(y)\mu_a}{1 + x_a(y)})}.
\end{align*}
\end{theorem}

\subsection{Computing $\boldsymbol{w}^*$ Efficiently}

Since both $g_a(x)$ and $F_{\boldsymbol{\mu}, \boldsymbol{c}}(y)$ functions are continuous, we can solve the zero points of $F_{\boldsymbol{\mu}, \boldsymbol{c}}(y) - 1$ and $g_a(x)-y^*$ efficiently, e.g., through bisection methods or Newton's iteration. We sum up the algorithm to compute $\boldsymbol{w}^*$ in Algorithm.~\ref{alg:opt-prop}. Here, the function $\mathsf{ZeroPoint}(\cdot)$ takes the input of a continuous monotonic function and outputs its zero point using an arbitrary method.
\begin{algorithm}
\caption{$\mathsf{ComputeProportions}(\boldsymbol{\mu},\boldsymbol{c})$} \label{alg:opt-prop}
\KwIn{expected reward $\boldsymbol{\mu}$; expected cost $\boldsymbol{c}$.}
    compute the maximum point $y^*$ of function $G(y)$ through its derivative's zero point: $y^* \gets \mathsf{ZeroPoint}\left( F_{\boldsymbol{\mu}, \boldsymbol{c}}(y) - 1 \right)$\;
    \For{$a \in [K]$}{
    compute the solution: $x_a^* \gets \mathsf{ZeroPoint}\left( g_a(x) - y^* \right)$
    }
    \For{$a \in [K]$}{
    transform back to the optimal proportion: $w_a^* = \frac{c_a x_a(y^*)}{c_1 + c_2 x_2(y^*) + \cdots + c_K x_K(y^*)}$\;
    }
    \Return optimal proportion $\boldsymbol{w}^* = (w_1^*,w_2^* , \cdots, w_K^*)$
\end{algorithm}

\subsection{Proofs for Lemma.~\ref{lemma:g-eq} and Theorem.~\ref{thm:opt-prop}}

\begin{proof}[\textbf{Proof of Lemma.~\ref{lemma:g-eq}}]
We now show that all the $g_a\left(x_a^*\right)$ have to be equal. Let
$\mathcal{B}$ be the set of arms that the optimal $g_b(x_b^*)$ function on the optimal solution $\boldsymbol{x}^*$ is not among the minimum of all arms, i.e., 
$$
\mathcal{B}=\left\{b \in\{2, \ldots, K\}: g_b\left(x_b^*\right)=\min _{a \neq 1} g_a\left(x_a^*\right)\right\}
$$
and $\mathcal{A}=\{2, \ldots, K\} \backslash \mathcal{B}$ to be its complement which achieves the minimum $g_a(x)$ function on $\boldsymbol{x}^*$. Assume $\mathcal{A} \neq \emptyset$. For all $a \in \mathcal{A}$ and $b \in \mathcal{B}$, one has $g_a\left(x_a^*\right)>g_b\left(x_b^*\right)$. Using the continuity of the $g$ functions and the fact that they are strictly increasing, there exists $\epsilon>0$ such that
$$
\left.\forall a \in \mathcal{A}, b \in \mathcal{B}, \quad g_a\left(x_a^*-\epsilon /c_a|\mathcal{A}|\right)\right)>g_b\left(x_b^*+\epsilon /c_b|\mathcal{B}|\right)>g_b\left(x_b^*\right)
$$
Introducing $\bar{x}_a=x_a^*-\epsilon /c_a|\mathcal{A}|$ for all $a \in \mathcal{A}$ and $\bar{x}_b=x_b^*+\epsilon /c_b|\mathcal{B}|$ for all $b \in \mathcal{B}$, there exists $b \in \mathcal{B}$ :
$$
\frac{\min _{a \neq 1} g_a\left(\bar{x}_a\right)}{c_1+c_2\bar{x}_2+\ldots + c_K\bar{x}_K}=\frac{g_b\left(x_b^*+\epsilon /c_b|\mathcal{B}|\right)}{c_1+c_2x_2^*+\cdots+c_K x_K^*}>\frac{g_b\left(x_b^*\right)}{c_1+c_2x_2^*+\cdots+c_Kx_K^*}=\frac{\min _{a \neq 1} g_a\left(x_a^*\right)}{c_1+c_2x_2^*+\cdots+c_K x_K^*},
$$
which contradicts the fact that $x^*$ belongs to the solution of the optimization problem \eqref{eq:g-func}. Hence $\mathcal{A}=\emptyset$ and there exists $y^* \in\left[0, d\left(\mu_1, \mu_2\right)\right)$ such that
$$
\forall a \in\{2, \ldots, K\}, g_a\left(x_a^*\right)=y^* \Leftrightarrow x_a^*=x_a\left(y^*\right).
$$
\end{proof}

\begin{proof}[\textbf{Proof of Theorem.~\ref{thm:opt-prop}}]
Recall the $G(y)$ function as follows:
\begin{align*}
    G(y) \coloneqq \frac{y}{c_1 + c_2 x_2(y) + \cdots + c_K x_K(y)}.
\end{align*}
From \eqref{eq:g-func} and Lemma \ref{lemma:g-eq}, we know that $y^* = \underset{y}{\operatorname{argmax}}\ G(y)$. Therefore, we can solve for $y^*$ analytically. From \cite{Garovoer16BAI}, we can take the derivative of $x_a(y)$ and obtain:
\begin{align*}
    x_a'(y) = 1/d(\mu_a, m_a(x_a(y))), \quad \text{ where } m_a(x) = \frac{\mu_1 + x\mu_a}{1 + x}
\end{align*}
Therefore, we can also compute the derivative of $G(y)$ as follows:
\begin{align*}
    G'(y) = \frac{(c_1 + c_2 x_2(y) + \cdots + c_K x_K(y)) - \sum_{a = 2}^{K} \frac{c_a y}{d(\mu_a, m_a(x_a(y))}}{(c_1 + c_2 x_2(y) + \cdots + c_K x_K(y))^2} .
\end{align*}
The condition $G'(y) = 0$ gives us that
\begin{align*}
    \sum_{a = 2}^{K} \frac{c_a y}{d(\mu_a, m_a(x_a(y))} &= c_1 + c_2 x_2(y) + \cdots + c_K x_K(y).
\end{align*}
Thus, $y^*$ is the solution to the following 
\begin{align*}
    \sum_{a = 2}^{K} \frac{c_a \cdot d(\mu_1, m_a(x_a(y))}{c_1 \cdot d(\mu_a, m_a(x_a(y))} &= 1,
\end{align*}
which is exactly Theorem.~\ref{thm:opt-prop} suggests. After obtaining $y^*$, we can recover $w_a^*$ as follows:
\begin{align*}
    w_1^* =& \frac{c_1}{c_1 + c_2 x_2(y^*) + \cdots + c_k x_k(y^*))}, \\
    w_a^* =& \frac{c_a x_a(y^*)}{c_1 + c_2 x_2(y^*) + \cdots + c_K x_K(y^*)}, \quad \forall a \in\{2,\cdots,K\}.
\end{align*}
\end{proof}

\section{Proof Roadmap of Theorem.~\ref{thm:exp-opt}}
We first show that the empirical cost proportion converges to the optimal proportion $\boldsymbol{w}^*$ in Section.~\ref{sec:as-conv} and in Theorem.~\ref{thm:as-conv}. Then, we prove a weaker version of Theorem.~\ref{thm:exp-opt} which states that the cost performance of $\mathsf{CTAS}$ matches the lower bound asymptotically with probability $1$, i.e., Theorem.~\ref{thm:as-upper} in Sec.~\ref{sec:as-upper}. Finally, we prove Theorem.~\ref{thm:exp-opt} in Sec.~\ref{sec:exp-opt}.

\section{Asymptotic Convergence of Cost Proportions $\boldsymbol{\hat{w}}(t)$ for $\mathsf{CTAS}$}\label{sec:as-conv}

\begin{theorem}\label{thm:as-conv}
Following the sampling rule of the $\mathsf{CTAS}$ algorithm in Algorithm.~\ref{alg:ctas}, we have:
\begin{align*}
    \mathbb{P}_{\boldsymbol{\mu}\times \boldsymbol{c}}\left(\lim_{t\rightarrow\infty}\ \frac{\widehat{c}_a(t) N_a(t)}{J(t)} = w_a^*\right) = 1.
\end{align*}
\end{theorem}

\subsection{Proof of Theorem \ref{thm:as-conv}}
We now provide a complete proof of Theorem.~\ref{thm:as-conv} which shows that the empirical cost proportion converges to the optimal proportion ${w}^*_a$ for all actions almost surely. We follow the notation that $\hat{w}_a(t) = \widehat{c}_a(t) N_a(t)/J(t)$ for all arm $a$ and let $\boldsymbol{w}^*(\boldsymbol{\mu},\boldsymbol{c}) = \{w^*_a(\boldsymbol{\mu},\boldsymbol{c})\}_{a\in\mathcal{A}}$ to be the output of Algorithm.~\ref{alg:opt-prop} when the input is $(\boldsymbol{\mu},\boldsymbol{c})$. First of all, it is easy to see that $\boldsymbol{w}^*(\boldsymbol{\mu},\boldsymbol{c})$ is a continuous function.
\begin{lemma}\label{lemma:cont-prop}
    For every $\boldsymbol{\mu} $ and $\boldsymbol{c}$, $\boldsymbol{w}^*(\boldsymbol{\mu},\boldsymbol{c})$ is continuous at $(\boldsymbol{\mu}, \boldsymbol{c})$.
\end{lemma}
\begin{proof}[\textbf{Proof of Lemma.~\ref{lemma:cont-prop}}]
Since the cost of each arm is bounded, we inherit the same properties of $F_{\boldsymbol{\mu}, \boldsymbol{c}}(y)$ as $F_{\boldsymbol{\mu}}(y)$ from \cite{Garovoer16BAI}. Namely, $\frac{\mathrm{d}}{\mathrm{d}y} F_{\boldsymbol{\mu}, \boldsymbol{c}}(y^*) \neq 0$. Since $y^*$ is solved by letting $F_{\boldsymbol{\mu}, \boldsymbol{c}}(y^*) = 1$, it is clear that $y^*$ is continuous in terms of the input $(\boldsymbol{\mu}, \boldsymbol{c})$. By composition, $x_a(y^*)$ is a continuous of $(\boldsymbol{\mu}, \boldsymbol{c})$ and consequently so is $\boldsymbol{w}^*(\boldsymbol{\mu},\boldsymbol{c})$ as desired.
\end{proof}

Let $\boldsymbol{\hat{w}}^*(t)=\{w_a^*(\boldsymbol{\hat{\mu}}(t), \boldsymbol{\hat{c}}(t))\}_{a\in\mathcal{A}}$ for simplicity. The following lemma shows that almost surely, $\boldsymbol{\hat{w}}^*(t)$ converges to $\boldsymbol{w}_a^*$ as $t\to \infty$, which is ensured by the forced exploration mechanism of $\mathsf{CTAS}$.
\begin{lemma}\label{lemma:w_conv}
    The sampling rule of the $\mathsf{CTAS}$ algorithm presented in algorithm.~\ref{alg:ctas} ensures $\mathbb{P}\left(\lim_{t\to \infty} \boldsymbol{\hat{w}}^*(t) = \boldsymbol{w}^* \right) = 1$.
\end{lemma}
\begin{proof}[\textbf{Proof of Lemma.~\ref{lemma:w_conv}}]
    We first show that the number of pulls $N_a(t)$ for all arms will increase to infinity as $t$ increases to infinity, when we ignore the stopping rule. Specifically, we first show that for any arm $a$ and any time $t$, $N_a(t) \geq \sqrt{t} - 1$. 
    
    For all positive integer $n$, we define integer $t_n=\inf\{t\in\mathbb{N}:\sqrt{t}\geq n\}$. We divide the time horizon into frames, and the $n$-th frame is denoted by $\mathcal{I}_n=\{t_n,\cdots,t_{n+1}-1\}$. When $n$ is large enough, we have $|\mathcal{I}_n|>K$. For all $t\in\mathcal{I}_n$, we have $n\leq \sqrt{t}<n+1$. According to the proof of Lemma.~17 from~\cite{Garovoer16BAI} and using a similar induction argument, it is easy to show that $\forall t \in \mathcal{I}_n$, $N_a(t)\geq n$. This is because the algorithm will be forced to pull arms which has not been pulled $n$ times in the first $K$ steps in this frame. Therefore, we have for $t\in\mathcal{I}_n$, $N_a(t)\geq n > \sqrt{t}-1$ for all action $a$. This indicates that $N_a(t) \geq \sqrt{t} - 1$ for all $t$ and $a$. Then, we have $N_a(t)\to \infty$ as $t\to \infty$.

    By the Strong Law of Large Numbers, we have that $\hat{\boldsymbol{\mu}}(t) \to \boldsymbol{\mu}$ almost surely and $\hat{\boldsymbol{c}}(t) \to \boldsymbol{c}$ almost surely as $N_a(t)$ increases to infinity. Let $\mathcal{E}$ be the event that:
    $\mathcal{E} = \{\hat{\boldsymbol{\mu}}(t) \to \boldsymbol{\mu}\} \cap \{\hat{\boldsymbol{c}}(t) \to \boldsymbol{c}\}$, we have $P(\mathcal{E}) = 1$. Then by continuity of $\boldsymbol{w}^*(\boldsymbol{\mu},\boldsymbol{c})$ from Lemma.~\ref{lemma:cont-prop}, we have $\boldsymbol{w}^*(t) \to \boldsymbol{w}^*$ for any sample path in $\mathcal{E}$. This means $\mathbb{P}\left(\lim_{t\to \infty} \boldsymbol{\hat{w}}^*(t) = \boldsymbol{w}^* \right) = 1$.
\end{proof}
Recall that our $\mathsf{CTAS}$ algorithm will pull the arm which has the largest cost deficit, so that it could contribute more cost and bring the empirical proportion $\boldsymbol{\hat{w}}(t)$ closer to the estimate $\boldsymbol{\hat{w}}^*(t)$. So when the estimate of estimate $\boldsymbol{\hat{w}}^*(t)$ is close enough to the optimal proportion $\boldsymbol{w}^*$, it should be intuitive that the empirical cost proportion $\boldsymbol{\hat{w}}(t)$ is also close to the optimal proportion due to the negative feedback mechanism. The next Lemma.~\ref{lemma:track} ensures this property, which is essential in proving Theorem.~\ref{thm:as-conv}.

\begin{lemma}\label{lemma:track}
For every $\epsilon > 0$, if there exists $t_0$ such that under the sampling rule of the $\mathsf{CTAS}$ algorithm, we have: 
\begin{align*}
    \max_{a} \left|\hat{w}_a^*(t) - w_a^*\right| < \epsilon, \quad \forall t\geq t_0,
\end{align*}
then, the $\mathsf{CTAS}$ algorithm will ensure that there exists a constant $t_\epsilon\geq t_0$ such that for all $t\geq t_\epsilon$:
\begin{align*}
    \left|\frac{\hat{c}_a(t) N_a(t)}{J(t)} - w_a^*\right| \leq 3(K - 1)\epsilon.
\end{align*}
\end{lemma}
\noindent Now we are ready to prove Theorem.~\ref{thm:as-conv} with the help of Lemma.~\ref{lemma:w_conv} and Lemma.~\ref{lemma:track}.

\begin{proof}[\textbf{Proof of Theorem \ref{thm:as-conv}}]
Let $\mathcal{E}$ be the set of sample paths such that $\{\lim_{t\to \infty} \boldsymbol{\hat{w}}^*(t) = \boldsymbol{w}^* \}$ holds. According to Lemma.~\ref{lemma:w_conv}, $\mathcal{E}$ holds almost surely and there exists a constant $t_0$ such that for every $\omega \in \mathcal{E}$ and $t\geq t_0$, we have:
\begin{align*}
    \max_{a} \left|w_a^*(\hat{\mu}(t)) - w_a^*(\mu)\right| < \frac{\epsilon}{3(K - 1)}.
\end{align*}
Therefore, using Lemma \ref{lemma:track}, there exists $t_\epsilon > t_0$ such that for every $t\geq t_\epsilon$, we have:
\begin{align*}
    \max_{a} \left|\frac{\hat{c}_a(t) N_a(t)}{J(t)} - w_a^*(\mu)\right| < \epsilon,
\end{align*}
which implies $\lim_{t\to\infty} \hat{w}_a(t) \to w^*_a$ for every $\omega \in \mathcal{E}$.
\end{proof}

\subsection{Proof of Lemma.~\ref{lemma:track}}
The proof of Lemma.~\ref{lemma:track} takes inspiration from the tracking results in.~\cite[Lemma.~8]{Garovoer16BAI}, which is motivated by~\cite{antos2008active}. However in our case, the total cost $J(t)$ and the empirical cost estimation $\hat{\boldsymbol{c}}(t)$ are both random variables which doesn't appear in the original proofs. It requires more delicate analysis.

\begin{proof}[\textbf{Proof of Lemma.~\ref{lemma:track}}]
Define $\mathcal{E}$ to be the set that $\boldsymbol{w}^*(t) \to \boldsymbol{w}^*$, $\boldsymbol{\hat{\mu}}(t) \to \boldsymbol{\mu}$, and $\boldsymbol{\hat{c}}(t)\to \boldsymbol{c}$ hold. And by Lemma.~\ref{lemma:w_conv} and the strong law of large numbers, $\mathcal{E}$ holds almost surely. Define $E_{a,t} \coloneqq \hat{c}_a(t) N_a(t) - J(t)w_a^*$ to be the cost overhead compared to the optimal proportion $\boldsymbol{w}^*$. It is easy to see that:
\begin{align*}
    \sum_{a = 1}^{K} E_{a,t} = \sum_{a = 1}^{K}\hat{c}_a(t) N_a(t) - \sum_{a = 1}^{K} J(t) w_a^* = J(t) - J(t) =  0.
\end{align*}
For each action $a$, we have $E_{a,t} \leq \max_{a}E_{a,t}$ on one hand. On the other hand, we can lower bound its overhead as follows:
\begin{align*}
    E_{a,t} =  - \sum_{a' \neq a} E_{a',t}  \geq -(K-1) \max_a E_{a,t}.
\end{align*}
Therefore, the maximum overhead in absolute value can be bounded as follows:
\begin{align*}
    \max_{a} |E_{a,t}| \leq (K - 1)\max_a E_{a,t}.
\end{align*}
Then, it suffice to bound $E_{a,t}$. By the boundedness of cost, there exists a constant $t_0$ such that for $t\geq t_0$, we have all the following hold:
\begin{align*}
    \sqrt{t} < 2\epsilon l t \leq  2J(t)\epsilon \leq \frac{2J(t)\epsilon}{\hat{c}_a(t)}, \quad \forall a\in\mathcal{A}; \quad
    \max_{a}|\hat{w}_a^*(t) - w_a^*| \leq \epsilon; \quad \max_a|\hat{\mu}_a(t) - \mu_a| \leq \epsilon; \quad 
    \max_{a} |\hat{c}_a(t) - c_a| \leq \epsilon. 
\end{align*}
This means $\hat{c}_a(t)\sqrt{t}\leq 2J(t) \epsilon$ for $t\geq t_0$. If at time $t$, the algorithm picks $A_{t+1} = a$ to explore, it means either this arm $a$ is under-explored, i.e., $N_a(t)\leq \sqrt{t}$, or the arm has the largest cost deficit, i.e., 
\begin{align*}
    a = \underset{a'}{\operatorname{argmax}}\ J(t) \hat{w}^*_{a'}(t) - \hat{c}_{a'}(t) N_{a'}(t).
\end{align*}
In the first case, we can bound the cost overhead as follows:
\begin{align*}
    E_{a, t} = \hat{c}_a(t) N_a(t) - J(t)w_a^* \leq \hat{c}_a(t) \sqrt{t} - J(t)w^*_a \leq \hat{c}_a(t) \sqrt{t} \leq 2J(t)\epsilon.
\end{align*}
In the second case, we have:
\begin{align*}
    \hat{c}_{a}(t) N_{a}(t) - J(t) \hat{w}^*_{a}(t) 
    =& \min_{a'} \hat{c}_{a'}(t) N_{a'}(t) - J(t) \hat{w}^*_{a'}(t) 
    = \min_{a} E_{a,t} + J(t)(\hat{w}_a^*(t) - w_a^*)\\
    \leq& \min_a E_{a,t} + J(t)\epsilon \\
    \leq& 2J(t)\epsilon, 
\end{align*}
where the first inequality holds for every sample path on $\mathcal{E}$, and the second inequality is due to $\min_{a}E_{a,t}\leq 0$. Therefore, we have $\{A_{t+1} = a\}\subset \{E_{a,t}\leq 2 J(t) \epsilon, \forall t\geq t_0\}$. We next show by induction that $E_{a,t}\leq \max\{E_{a,t_0}, 2J(t)\epsilon +1\}$ for all $t\geq t_0$. For our base case, let $t = t_0$, and the statement clearly holds. So let $t\geq t_0$ such that we assume $E_{a,t}\leq \max\{E_{a,t_0}, 2J(t)\epsilon +1 \}$, then if $E_{a,t} \leq 2J(t)\epsilon $, we have:
\begin{align*}
    E_{a, t + 1} =& E_{a, t} + C(t+1) \mathbf{1}_{\{A_{t+1} = a\}} - C(t+1) w_a^* 
    \leq E_{a, t} + C(t+1) \mathbf{1}_{\{E_{a,t}\leq 2 J(t) \epsilon\}} - C(t+1) w_a^* \\
    =& E_{a, t} + C(t+1) (1- w_a^*) 
    \leq 2J(t)\epsilon +1 \leq 2J(t+1)\epsilon +1 
    \leq \max\{E_{a,t_0}, 2J(t+1)\epsilon+1\}, 
\end{align*}
where the first inequality is due to $\{A_{t+1} = a\}\subset \{E_{a,t}\leq 2 J(t) \epsilon, \forall t\geq t_0\}$, and the second inequality uses the induction assumption. If $E_{a,t}>2J(t)\epsilon$, the indicator is zero and we have:
\begin{align*}
    E_{a,t+1} \leq E_{a,t} - C(t+1)w_a^* \leq \max\{E_{a,t_0}, 2J(t)\epsilon +1 \} - C(t+1)w_a^* \leq \max\{E_{a,t_0}, 2J(t+1)\epsilon +1 \},
\end{align*}
which concludes the induction that $E_{a,t}\leq \max\{E_{a,t_0}, 2J(t)\epsilon +1\}$ for all $t\geq t_0$. Substitute the definition of $E_{a,t}$ in and we will have $\hat{c}_a(t) N_a(t) - J(t)w_a^* \leq \max\{E_{a,t_0}, 2J(t)\epsilon +1\}$, which indicates:
\begin{align*}
    \frac{\hat{c}_a(t) N_a(t)}{J(t)} - w_a^* &\leq \max\left(\frac{E_{a,t_0}}{J(t)}, 2\epsilon + \frac{1}{J(t)}\right) 
    \leq \max\left(\frac{t_0}{J(t)}, 2\epsilon + \frac{1}{J(t)}\right).
\end{align*}
Since $J(t)\to \infty$ as $t$ increases, there exists a constant $t_\epsilon\geq t_0$ such that for any $t\geq t_\epsilon$, we have:
\begin{align*}
    \frac{\hat{c}_a(t) N_a(t)}{J(t)} - w_a^* &
    \leq \max\left(\frac{t_0}{J(t)}, 2\epsilon + \frac{1}{J(t)}\right) \leq 3\epsilon.
\end{align*}
Therefore, for every $t\geq t_\epsilon$, we have:
\begin{align*}
    \max_{a}\left|\frac{\hat{c}_a(t) N_a(t)}{C(t)} - w_a^*\right| = \max_{a} \frac{|E_{a,t}|}{J(t)} \leq (K-1)\max_a \frac{E_{a,t}}{J(t)}= (K-1)\max_a \left( \frac{\hat{c}_a(t) N_a(t)}{J(t)} - w_a^*  \right) \leq 3(K - 1)\epsilon.
\end{align*}

\end{proof}

\section{Asymptotic Cumulative Cost Upper Bound for $\mathsf{CTAS}$}\label{sec:as-upper}

\begin{theorem}[Almost Sure Cost Upper Bound]\label{thm:as-upper}
Let $\delta \in [0,1)$ and $\alpha \in[1, e / 2]$. Using the Chernoff's stopping rule with $\beta(t, \delta)=\log (\mathcal{O}(t^\alpha) / \delta)$, the $\mathsf{CTAS}$ algorithm ensures:
$$
\mathbb{P}_{\boldsymbol{\mu} \times \boldsymbol{c}}\left(\limsup _{\delta \rightarrow 0} \frac{J(\tau_\delta)}{\log (1 / \delta)} \leq \alpha T^*(\boldsymbol{\mu})\right)=1.
$$
\end{theorem}

In this section, we give a proof of Theorem.~\ref{thm:as-upper} and Theorem.~\ref{thm:exp-opt} which characterizes the upper bound of cumulative cost asymptotically. Before proving the Theorems, we present the following technical lemma which is useful in these proofs and can be checked easily.
\begin{lemma}\label{lemma:tech-bound}
For every $\alpha \in[1, e / 2]$, for any two constants $c_1, c_2>0$,
$$
x=\frac{\alpha}{c_1}\left[\log \left(\frac{c_2 e}{c_1^\alpha}\right)+\log \log \left(\frac{c_2}{c_1^\alpha}\right)\right]
$$
is such that $c_1 x \geq \log \left(c_2 x^\alpha\right)$.
\end{lemma}

\subsection{Proof of Theorem.~\ref{thm:as-upper}}
In this section, we first prove the almost sure upper bound from Theorem.~\ref{thm:as-upper} with the help of Theorem.~\ref{thm:as-conv}.
\begin{proof}[\textbf{Proof of Theorem.~\ref{thm:as-upper}}]
Let $\mathcal{E}$ be the event that all concentrations regarding reward, cost, and empirical proportion holds, i.e.,
$$
\mathcal{E}=\left\{\forall a \in \mathcal{A}, \frac{c_a N_a(t)}{J(t)} \underset{t \rightarrow \infty}{\rightarrow} w_a^*\right\} \cap \left\{\hat{\boldsymbol{\mu}}(t) \underset{t \rightarrow \infty}{\rightarrow} \boldsymbol{\mu}\right\} \cap \left\{\hat{\boldsymbol{c}}(t) \underset{t \rightarrow \infty}{\rightarrow} \boldsymbol{c}\right\}
$$
By Theorem.~\ref{thm:as-conv} and the law of large numbers, $\mathcal{E}$ is of probability 1. On $\mathcal{E}$, there exists $t_0$ such that for all $t \geq t_0, \hat{\mu}_1(t)>\max _{a \neq 1} \hat{\mu}_a(t)$ due to the concentration of empirical reward, and thus the Chernoff stopping statistics can be re-written as:
$$
\begin{aligned}
 Z(t)=&\min _{a \neq 1} Z_{1, a}(t)=\min _{a \neq 1} N_1(t) d\left(\hat{\mu}_1(t), \hat{\mu}_{1, a}(t)\right)+N_a(t) d\left(\hat{\mu}_a(t), \hat{\mu}_{1, a}(t)\right) \\
 = &J(t) \cdot \left[\min _{a \neq 1}\left(\frac{N_1(t)}{J(t)}+\frac{N_a(t)}{J(t)}\right) I_{\frac{N_1(t) / J(t)}{N_1(t) / J(t) +N_a(t) / J(t)}}\left(\hat{\mu}_1(t), \hat{\mu}_a(t)\right)\right] .
\end{aligned}
$$
By Lemma \ref{lemma:cont-prop}, for all $a \geq 2$, the mapping $(\boldsymbol{w}, \boldsymbol{\lambda}, \boldsymbol{c}) \rightarrow\left(\frac{w_1}{c_1} + \frac{w_a}{c_a}\right) I_{\frac{w_1 / c_1}{\left(w_1 / c_1 + w_a / c_a\right)}}\left(\lambda_1, \lambda_a\right)$ is continuous at $\left(w^*(\boldsymbol{\mu}), \boldsymbol{\mu}, \boldsymbol{c}\right)$. Therefore, for all $\epsilon>0$ there exists $t_1 \geq t_0$ such that for all $t \geq t_1$ and all $a \in\{2, \ldots, K\}$,
$$
\left(\frac{N_1(t)}{C(t)}+\frac{N_a(t)}{C(t)}\right) I_{\frac{N_1(t) / C(t)}{N_1(t) / C(t) +N_a(t) / C(t)}}\left(\hat{\mu}_1(t), \hat{\mu}_a(t)\right) \geq \frac{w_1^*/c_1 +w_a^*/c_a}{1+\epsilon} I_{\frac{w_1^*/c_1}{w_1^*/c_1+w_a^*/c_a}}\left(\mu_1, \mu_a\right)
$$
Hence, for any $t \geq t_1$, we have:
$$
Z(t) \geq J(t)\min _{a \neq 1}\frac{w_1^*/c_1 +w_a^*/c_a}{1+\epsilon} I_{\frac{w_1^*/c_1}{w_1^*/c_1+w_a^*/c_a}}\left(\mu_1, \mu_a\right)=\frac{J(t)}{(1+\epsilon) T^*(\boldsymbol{\mu})}.
$$
Consequently, we can bound the cumulative cost a stopping time $\tau_\delta$ as follows:
$$
\begin{aligned}
J(\tau_\delta)  =& J \Big(\inf \{t \in \mathbb{N}: Z(t) \geq \beta(t, \delta)\}\Big) 
 \leq J(t_1) \vee J\Big(\inf \left\{t \in \mathbb{N}: J(t)(1+\epsilon)^{-1} T^*(\boldsymbol{\mu})^{-1} \geq \log (Bt^\alpha / \delta)\right\}\Big)\\
 \leq & J(t_1) \vee (1+\epsilon)T^*(\boldsymbol{\mu})\left(\log \left(\frac{B}{\ell^\alpha\delta} \right) + \alpha \log(J(\tau_\delta)) \right) + \mathcal{O}(1),
\end{aligned}
$$
for some positive constant $B$ from Proposition ~\ref{prop:beta}. Using Lemma \ref{lemma:tech-bound}, it follows that on $\mathcal{E}$, for $\alpha \in[1, e / 2]$
$$
J(\tau_\delta) \leq J(t_1) \vee \alpha(1+\epsilon) T^*(\boldsymbol{\mu})\left[\log \left(\frac{B e\left((1+\epsilon) T^*(\boldsymbol{\mu})\right)^\alpha}{\ell^{\alpha} \cdot \delta}\right)+\log \log \left(\frac{B\left((1+\epsilon) T^*(\boldsymbol{\mu})\right)^\alpha}{\ell^{\alpha} \cdot \delta}\right)\right] + \mathcal{O}(1).
$$
Thus we have:
$$
\limsup _{\delta \rightarrow 0} \frac{J(\tau_\delta)}{\log (1 / \delta)} \leq(1+\epsilon) \alpha T^*(\boldsymbol{\mu}).
$$
Letting $\epsilon$ go to zero concludes the proof.
\end{proof}

\subsection{Asymptotic Expectation Optimality}
In this section, we provide the proof of Theorem.~\ref{thm:exp-opt} which characterizes the expected cumulative cost upper bound for $\mathsf{CTAS}$.
\begin{proof}[\textbf{Proof of Theorem \ref{thm:exp-opt}}]
Without loss of generality, we assume that for every $a \in \{1, \ldots, K\}$, $c_a \in [\ell, 1]$ with $\ell > 0$. To ease the notation, we assume that the bandit model $\mu$ is such that $\mu_1>\mu_2 \geq \cdots \geq \mu_K$. Let $\epsilon>0$. From the continuity of $w^*$ in $\boldsymbol{\mu}$, there exists two continuous functions $\alpha(\epsilon)$ and $\beta(\epsilon)$ with $\lim_{\epsilon\rightarrow0}\alpha(\epsilon) = 0$ and $\lim_{\epsilon\rightarrow0}\beta(\epsilon) = 0$, and we have $\alpha(\epsilon) \leq\left(\mu_1-\mu_2\right) / 4$ such that
\[
\mathcal{I}_\epsilon:=\left[\mu_1-\alpha(\epsilon), \mu_1+\alpha(\epsilon)\right] \times \cdots \times\left[\mu_K-\alpha(\epsilon), \mu_K+\alpha(\epsilon)\right],
\]
and $\beta(\epsilon)$ small enough so that
\[
\mathcal{J}_\epsilon:=\left[c_1-\beta(\epsilon), c_1+\beta(\epsilon)\right] \times \cdots \times\left[c_K-\beta(\epsilon), c_K+\beta(\epsilon)\right]
\]
is such that for all $(\boldsymbol{\mu}^{\prime}, \boldsymbol{c}^{\prime}) \in \mathcal{I}_\epsilon \times \mathcal{J}_{\epsilon}$,
$$
\max _a\left|w_a^*\left(\boldsymbol{\mu}^{\prime}, \boldsymbol{c}^{\prime}\right)-w_a^*(\boldsymbol{\mu}, \boldsymbol{c})\right| \leq \epsilon
$$
In particular, whenever $\hat{\boldsymbol{\mu}}(t) \in \mathcal{I}_\epsilon$, the empirical best arm is $\hat{a}_t=1$.
Let $T \in \mathbb{N}$ and define $h(T):=T^{1 / 4}$ and the event
$$
\mathcal{E}_T(\epsilon)=\bigcap_{t=h(T)}^T\left\{ \hat{\boldsymbol{\mu}}(t) \in \mathcal{I}_\epsilon \right\} \cap  \left\{ \hat{\boldsymbol{c}}(t) \in \mathcal{J}_\epsilon\right\}
$$
We first present a lemma showing that the event $\mathcal{E}_T(\epsilon)$ is a high probability event as follows:
\begin{lemma}\label{lemma:neg-event}
There exists constants $B$ and $C$ such that
\begin{align*}
    \mathbb{P}(\mathcal{E}_T^{c}) \leq BT\exp(-CT^{1/8}).
\end{align*}
\end{lemma}
The proof of this lemma will be delayed. We now proceed in proving Theorem.~\ref{thm:exp-opt}. By Lemma \ref{lemma:track}, we have that there exists some $T_{\epsilon}$ such that for $T \geq T_{\epsilon}$, it holds on $\mathcal{E}_T$ that
\begin{align*}
    \forall t \geq \sqrt{T},\ \max_{a} \left|\frac{\hat{c}_a(t) N_a(t)}{J(t)} - w_a^*(\boldsymbol{\mu},\boldsymbol{c})\right| \leq 3(K - 1)\epsilon
\end{align*}
On the event $\mathcal{E}_T$, it holds for $t \geq h(T)$ that $\hat{a}_t=1$ and the Chernoff stopping statistic rewrites
$$
\begin{aligned}
\max _a \min _{b \neq a} Z_{a, b}(t) & =\min _{a \neq 1} Z_{1, a}(t)=\min _{a \neq 1} N_1(t) d\left(\hat{\mu}_1(t), \hat{\mu}_{1, a}(t)\right)+N_a(t) d\left(\hat{\mu}_a(t), \mu_{1, a}(t)\right) \\
& = J(t) \cdot \left[\min _{a \neq 1}\left(\frac{N_1(t)}{J(t)}+\frac{N_a(t)}{J(t)}\right) I_{\frac{N_1(t) / J(t)}{N_1(t) / J(t)+N_a(t) / J(t)}}\left(\hat{\mu}_1(t), \hat{\mu}_a(t)\right)\right] \\
& =J(t) \cdot \operatorname{g}\left(\hat{\boldsymbol{\mu}}(t),\left(\frac{\hat{c}_a(t) N_a(t)}{J(t)}\right)_{a=1}^K\right),
\end{aligned}
$$
where we introduce the function
$$
g\left(\boldsymbol{\mu}^{\prime}, \boldsymbol{c}^{\prime}, \boldsymbol{w}^{\prime}\right)=\min _{a \neq 1}\left(\frac{w_1^{\prime}}{c_1^{\prime}}+\frac{w_a^{\prime}}{c_a^{\prime}}\right) I_{\frac{w_1^{\prime}/c_1^{\prime}}{w_1^{\prime}/c_1^{\prime}+w_a^{\prime}/c_a^{\prime}}}\left(\mu_1^{\prime}, \mu_a^{\prime}\right)
$$
Using Lemma \ref{lemma:track}, for $T \geq T_\epsilon$, introducing
$$
C_\epsilon^*(\boldsymbol{\mu}, \boldsymbol{c})=\inf _{\substack{\boldsymbol{\mu}^{\prime}:\left\|\boldsymbol{\mu}^{\prime}-\boldsymbol{\mu}\right\| \leq \alpha(\epsilon)\\ \boldsymbol{c}^{\prime}: \left\|\boldsymbol{c}^{\prime}-\boldsymbol{c}\right\| \leq \beta(\epsilon) \\ \boldsymbol{w}^{\prime}:\left\|\boldsymbol{w}^{\prime}-w^*(\boldsymbol{\mu},\boldsymbol{c})\right\| \leq 3(K-1) \epsilon}} 
g\left(\boldsymbol{\mu}^{\prime}, \boldsymbol{c}^{\prime},\boldsymbol{w}^{\prime}\right),
$$
where $\alpha(\epsilon)$ and $\beta(\epsilon)$ are two continuous functions such that $\lim_{\epsilon\rightarrow0}\alpha(\epsilon) = 0$ and $\lim_{\epsilon\rightarrow0}\beta(\epsilon) = 0$.
On the event $\mathcal{E}_T$, it holds that for every $t \geq \sqrt{T}$, we have:
$$
\max _a \min _{b \neq a} Z_{a, b}(t) \geq J(t) \cdot C_\epsilon^*(\boldsymbol{\mu},\boldsymbol{c}).
$$
Let $T \geq T_{\epsilon}$. On $\mathcal{E}_T$, we have:
\begin{align*}
    \min(J(\tau_{\delta}), J(T)) 
    \leq & J(\sqrt{T}) + \sum_{t = \sqrt{T}}^T C_t \cdot \mathbf{1}_{(\tau_\delta > t)}
    \leq J(\sqrt{T}) + \sum_{t = \sqrt{T}}^T C_t \cdot  \mathbf{1}_{ (\max _a \min _{b \neq a} Z_{a, b}(t) \leq \beta(T,\delta) )}\\
    \leq & J(\sqrt{T} ) + \sum_{t = \sqrt{T}}^T C_t \cdot \mathbf{1}_{(J(t) \cdot C_{\epsilon}^*(\boldsymbol{\mu},\boldsymbol{c}) \leq \beta(T, \delta))} 
    \leq J(\sqrt{T}) + \frac{\beta(T, \delta)}{C_{\epsilon}^{*}(\boldsymbol{\mu},\boldsymbol{c})}.
\end{align*}
For every sample path $\omega$, consider $T_0(\delta, \omega) \coloneqq \inf \left\{T \mid J(\sqrt{T}) + \frac{\beta(T, \delta)}{C_{\epsilon}^{*}(\boldsymbol{\mu},\boldsymbol{c})} \leq J(T, \omega) \right\}$.  Then for every $T \geq \max(T_0(\delta, \omega), T_{\epsilon})$, that is for every $T$ such that $J(T,\omega)\geq \max\{J(T_0(\delta,\omega)), T_\epsilon\}$, we have that  $\min\{J(\tau_{\delta}, \omega), J(T, \omega)\} < J(T, \omega)$, and thus $J(\tau_\delta,w) < J(T,w) $. So it means $J(T_0(\delta,\omega))+T_\epsilon$ is an upper bound of $J(\tau_\delta,\omega) $ and next we intend to bound $J(T_0(\delta, \omega))$ for every $\omega \in \mathcal{E}_T$. Consider
\begin{align*}
    F(\eta, \omega) = \inf \left\{T \in \mathbb{N} \mid J(T, \omega) - J(\sqrt{T}) \geq J(T, \omega) / (1 + \eta)\right\} = \inf \left\{T \in \mathbb{N} \mid \frac{\eta}{1+\eta} J(T, \omega) \geq J(\sqrt{T}) \right\}.
\end{align*}
Then it is easy to see that
\begin{align*}
    F(\eta, \omega) \leq \inf \left\{T \in \mathbb{N} \mid \frac{\eta \ell T}{1+\eta} \geq \sqrt{T} \right\} \eqqcolon F'(\eta).
\end{align*}
Therefore,we have by the boundedness of cost:
\begin{align*}
    J(T_0(\delta , \omega)) 
    &\leq J(F(\eta)) + J\left(\inf \left\{T \mid \frac{1}{C_{\epsilon}^*(\boldsymbol{\mu},\boldsymbol{c})}\log\left(\frac{BT^{\alpha}}{\delta}\right) \leq \frac{J(T, \omega)}{1 + \eta}\right\} \right) \\
    &\leq J(F'(\eta)) + J\left( \inf \left\{T \mid \frac{J(T, \omega) C_{\epsilon}^*(\boldsymbol{\mu},\boldsymbol{c})}{1 + \eta} \geq \log\left(\frac{BT^{\alpha}}{\delta}\right) \right\} \right)
\end{align*}
where $B$ is  some constant from Proposition \ref{prop:beta}.
Define
\begin{align*}
   T_1(\delta, \omega) \coloneqq \inf \left\{T \mid \frac{J(T,\omega) C_{\epsilon}^*(\boldsymbol{\mu},\boldsymbol{c})}{1 + \eta} \geq \log\left(\frac{BT^{\alpha}}{\delta}\right) \right\} 
\end{align*}
By Lemma \ref{lemma:tech-bound}, we have:
\begin{align*}
    J\left(T_1(\delta, \omega) \right) 
    \leq& \frac{1+\eta}{C_\epsilon^*(\boldsymbol{\mu},\boldsymbol{c})} \log\left(\frac{BT_1(\delta,\omega)^{\alpha}}{\delta}\right) + \mathcal{O}(1) 
    \leq \frac{1+\eta}{C_\epsilon^*(\boldsymbol{\mu},\boldsymbol{c})} \log\left(\frac{B}{\delta \ell^\alpha} + \alpha \log \left( J(T_1(\delta,\omega))\right)\right) + \mathcal{O}(1) \\
    \leq & \frac{\alpha(1 + \eta)}{C_{\epsilon}^*(\boldsymbol{\mu},\boldsymbol{c})} \left[ \log\left(\frac{Be(1+\eta)^{\alpha}}{\delta  \ell^{\alpha} C_{\epsilon}^{*}(\boldsymbol{\mu},\boldsymbol{c})^{\alpha}}\right) + \log\log\left(\frac{B(1+\eta)^{\alpha}}{\delta \ell^{\alpha} C_{\epsilon}^*(\boldsymbol{\mu},\boldsymbol{c})^{\alpha}}\right)\right] + \mathcal{O}(1).
\end{align*}
Therefore, for every $\omega \in \mathcal{E}_T$, we have:
\begin{align*}
J(T_0(\delta, \omega), \omega) &\leq J(F'(\eta)) + J(T_1(\delta, \omega)) \\
&\leq \underbrace{\frac{\alpha(1 + \eta)}{C_{\epsilon}^*(\boldsymbol{\mu},\boldsymbol{c})} \left[ \log\left(\frac{Be(1+\eta)^{\alpha}}{\delta  \ell^{\alpha} C_{\epsilon}^{*}(\boldsymbol{\mu},\boldsymbol{c})^{\alpha}}\right) + \log\log\left(\frac{B(1+\eta)^{\alpha}}{\delta \ell^{\alpha} C_{\epsilon}^*(\boldsymbol{\mu},\boldsymbol{c})^{\alpha}}\right)\right]}_{\overline{C}(\eta, \epsilon)} + \mathcal{O}(1).
\end{align*}

Therefore, $\mathcal{E}_{T} \subset \{J(\tau_\delta) \leq \overline{C}(\eta, \epsilon) + T_\epsilon\}$ for every $T\geq \overline{C}(\eta, \epsilon) + T_\epsilon$. By definition of expectation, we have:
\begin{equation}\label{eq:expectation}
\begin{split}
    \mathbb{E}[J(\tau_{\delta})]
    = \int_{\Omega} J(\tau_{\delta}, \omega) \mathrm{d}\mathbb{P}  
    \leq& \int_{J(\tau_\delta) \leq \overline{C}(\eta, \epsilon) + T_\epsilon} J(\tau_{\delta}, \omega) \mathrm{d}\mathbb{P} + \int_{J(\tau_\delta) > \overline{C}(\eta, \epsilon) + T_\epsilon} J(\tau_{\delta}, \omega)\mathrm{d}\mathbb{P} \\
    \leq &  T_\epsilon + \overline{C}(\eta, \epsilon) + \sum_{T = T_\epsilon + \overline{C}(\eta, \epsilon)}^\infty \mathbb{P}\left( J(\tau_{\delta}) > T\right)\\
    \leq & T_\epsilon + \overline{C}(\eta, \epsilon) + \sum_{T = 1}^\infty \mathbb{P}\left(\mathcal{E}_{T}^c \right)
\end{split}
\end{equation}

Then by Lemma \ref{lemma:neg-event}, there exists two constants $B$ and $C$ such that we have:
\begin{align*}
    \sum_{T = 1}^\infty \mathbb{P}\left(\mathcal{E}_{T}^c \right) \leq \sum_{T=1}^\infty B T \exp(-CT^{1/8}) .
\end{align*}
So we have a bound that:
\begin{align*}
    \mathbb{E}[J(\tau_{\delta})] &\leq T_\epsilon + \frac{\alpha(1 + \eta)}{C_{\epsilon}^*(\boldsymbol{\mu},\boldsymbol{c})} \left[ \log\left(\frac{Be(1+\eta)^{\alpha}}{\delta \ell^{\alpha}  (C_{\epsilon}^{*}(\boldsymbol{\mu},\boldsymbol{c}))^{\alpha}}\right) + \log\log\left(\frac{B(1+\eta)^{\alpha}}{\delta \ell^{\alpha} (C_{\epsilon}^*(\boldsymbol{\mu},\boldsymbol{c}))^{\alpha}}\right)\right] + \sum_{T = 1}^{\infty} BT\exp(-CT^{1/8}) + \mathcal{O}(1).
\end{align*}
Thus, by dividing $\log(1/\delta)$ and let $\delta$ decreases to $0$, we have:
\begin{align*}
    \liminf_{\delta \to 0} \frac{\mathbb{E}_{\mu}[J(\tau_{\delta})]}{\log(1/\delta)} \leq \frac{\alpha(1 + \eta)}{C_{\epsilon}^*(\boldsymbol{\mu},\boldsymbol{c})}
\end{align*}
Letting $\eta$ and $\epsilon$ go to zero, it is easy to see that $C_\epsilon^*(\boldsymbol{\mu},\boldsymbol{c})$ converges to $T^*(\boldsymbol{\mu})$ we have that
\begin{align*}
    \liminf_{\delta \to 0} \frac{\mathbb{E}_{\mu}[J(\tau_{\delta})]}{\log(1/\delta)} \leq \alpha T^*(\boldsymbol{\mu}).
\end{align*}
\end{proof}

\begin{proof}[\textbf{Proof of Lemma \ref{lemma:neg-event}}]
By a union bound:
\begin{align*}
    \mathbb{P}(\mathcal{E}_T^{c}) 
    \leq \sum_{t= h(T)}^T\mathbb{P}(\hat{\boldsymbol{\mu}}(t) \notin \mathcal{I}_{\epsilon}) + \sum_{t= h(T)}^T\mathbb{P}(\hat{\boldsymbol{c}}(t) \notin \mathcal{J}_{\epsilon})
\end{align*}
From \cite[Lemma 19]{Garovoer16BAI}, there exists $B$ and $C$ such that
\begin{align*}
    \sum_{t= h(T)}^T\mathbb{P}(\hat{\boldsymbol{\mu}}(t) \notin \mathcal{I}_{\epsilon}) \leq BT\exp(-CT^{1/8}).
\end{align*}
So it is sufficient to bound the second term concerning the concentration of costs. We have:
\begin{align*}
    \sum_{t= h(T)}^T\mathbb{P}(\hat{\boldsymbol{c}}(t) \notin \mathcal{J}_{\epsilon}) = \sum_{t= h(T)}^T \sum_{a=1}^K\left[ \mathbb{P}(\hat{c}_a(t) \leq c_a - \beta) + \mathbb{P}(\hat{c}_a(t)\geq c_a + \beta) \right].
\end{align*}
Let $T$ be large enough such that $h(T) \geq K^2$. For any $t\geq h(T)$, we have $N_a(t)\geq \sqrt{t} - K$ for every arm $a$. Let $\hat{c}_{a,s}$ be the empirical mean of the first $s$ costs from arm $a$ such that $\hat{c}_a(t) = \hat{c}_{a,N_a(t)}$. With a union bound, we have:
\begin{align*}
    \mathbb{P}\left( \hat{c}_a(t) \leq c_a - \beta \right) 
    =& \mathbb{P}\left(\hat{c}_a(t) \leq c_a - \beta, N_a(t) \geq \sqrt{t} - K \right) 
    \leq \sum_{s=\sqrt{t} - K}^t\mathbb{P}(\hat{c}_{a,s}\leq c_a - \beta)\\
    \leq& \sum_{s=\sqrt{t} - K}^t\exp\left( -s\beta^2 \right) 
    \leq \frac{\exp\left(-(\sqrt{t} - K)\beta^2 \right)}{1-\exp(-\beta^2)},
\end{align*}
where the second last inequality uses Hoeffding's inequality. With a same argument, we can show that the same upper bound applies to $\mathbb{P}\left( \hat{c}_a(t) \geq c_a + \beta \right) $. So we can plug in and have:
\begin{align*}
    \sum_{t= h(T)}^T\mathbb{P}(\hat{\boldsymbol{c}}(t) \notin \mathcal{J}_{\epsilon}) \leq& 2\sum_{t= h(T)}^T \sum_{a=1}^K \frac{\exp\left(-(\sqrt{t} - K)\beta^2 \right)}{1-\exp(-\beta^2)} 
    \leq \frac{2KT\exp(K\beta^2)}{1-\exp(-\beta^2)}\exp\left(-\sqrt{h(T)}\beta^2 \right)\\
    \leq & \frac{2KT\exp(K\beta^2)}{1-\exp(-\beta^2)}\exp\left(-T^{1/8}\beta^2 \right).
\end{align*}
Finally, we re-define the constants $B$ and $C$ to be:
\begin{align*}
    B \leftarrow & 2\max\{B, \frac{2KT\exp(K\beta^2)}{1-\exp(-\beta^2)}\},\\
    C \leftarrow & \min\{ C, \beta^2\}.
\end{align*}
So we have:
\begin{align*}
    \mathbb{P}(\mathcal{E}_T^{c}) 
    \leq \sum_{t= h(T)}^T\mathbb{P}(\hat{\boldsymbol{\mu}}(t) \notin \mathcal{I}_{\epsilon}) + \sum_{t= h(T)}^T\mathbb{P}(\hat{\boldsymbol{c}}(t) \notin \mathcal{J}_{\epsilon})
    \leq BT\exp(-CT^{1/8}).
\end{align*}
\end{proof}

\section{Asymptotic Cost Optimality for Chernoff-Overlap in Two-armed Gaussian Bandits}\label{sec:exp-opt}
To get meaningful bounds related to the lower bound, we will consider the special cases when the lower bound is known, specifically the two-armed Gaussian bandit model. In order to prove Theorem.~\ref{thm:chernoff-two-arm}, we require a lemma similar to.~\ref{thm:as-conv} which studyies the convergence of empirical cost proportion. Therefore, we present Lemma.~\ref{lemma:asconv-co}.

\begin{lemma}\label{lemma:asconv-co}
 Under the Chernoff-Overlap sampling rule, we have that
\begin{align*}
\mathbb{P}\left( \lim_{t\to\infty} \frac{\widehat{c}_a(t) N_a(t)}{J(t)} = \frac{\sqrt{c_a}}{\sqrt{c_1} + \sqrt{c_2}}, \forall a\in\mathcal{A} \right) = 1
\end{align*}
\end{lemma}
\noindent Now, we can prove Theorem.~\ref{thm:chernoff-two-arm} following the similar argument as the proof of Theorem.~\ref{thm:as-upper}.
\begin{proof}[\textbf{Proof of Theorem.~\ref{thm:chernoff-two-arm}}]
Let $\mathcal{E}$ be the event that all concentrations regarding reward, cost, and empirical proportion holds, i.e.,
$$
\mathcal{E}=\left\{\forall a \in \mathcal{A}, \frac{c_a N_a(t)}{J(t)} \underset{t \rightarrow \infty}{\rightarrow} w_a^*\right\} \cap \left\{\hat{\boldsymbol{\mu}}(t) \underset{t \rightarrow \infty}{\rightarrow} \boldsymbol{\mu}\right\} \cap \left\{\hat{\boldsymbol{c}}(t) \underset{t \rightarrow \infty}{\rightarrow} \boldsymbol{c}\right\}
$$
By Lemma.~\ref{lemma:asconv-co} and the law of large numbers, $\mathcal{E}$ is of probability 1. On $\mathcal{E}$, there exists $t_0$ such that for all $t \geq t_0, \hat{\mu}_1(t)> \hat{\mu}_2(t)$ due to the concentration of empirical reward, and thus the Chernoff stopping statistics can be re-written as:
$$
\begin{aligned}
    Z_{2}(t)=&N_1(t) d\left(\hat{\mu}_1(t), \hat{\mu}_{1, 2}(t)\right)+N_2(t) d\left(\hat{\mu}_2(t), \hat{\mu}_{1, 2}(t)\right) \\
    = &J(t) \cdot \left[\left(\frac{N_1(t)}{J(t)}+\frac{N_2(t)}{J(t)}\right) I_{\frac{N_1(t) / J(t)}{N_1(t) / J(t) +N_2(t) / J(t)}}\left(\hat{\mu}_1(t), \hat{\mu}_2(t)\right)\right] .
\end{aligned}
$$
By Lemma \ref{lemma:cont-prop}, for all $a \geq 2$, the mapping $(\boldsymbol{w}, \boldsymbol{\lambda}, \boldsymbol{c}) \rightarrow\left(\frac{w_1}{c_1} + \frac{w_2}{c_2}\right) I_{\frac{w_1 / c_1}{\left(w_1 / c_1 + w_2 / c_2\right)}}\left(\lambda_1, \lambda_2\right)$ is continuous at $\left(w^*(\boldsymbol{\mu}), \boldsymbol{\mu}, \boldsymbol{c}\right)$. Therefore, for all $\epsilon>0$ there exists $t_1 \geq t_0$ such that for all $t \geq t_1$ and all $a \in\{2, \ldots, K\}$,
$$
\left(\frac{N_1(t)}{C(t)}+\frac{N_2(t)}{C(t)}\right) I_{\frac{N_1(t) / C(t)}{N_1(t) / C(t) +N_2(t) / C(t)}}\left(\hat{\mu}_1(t), \hat{\mu}_2(t)\right) \geq \frac{w_1^*/c_1 +w_2^*/c_2}{1+\epsilon} I_{\frac{w_1^*/c_1}{w_1^*/c_1+w_2^*/c_2}}\left(\mu_1, \mu_2\right)
$$
Hence, for any $t \geq t_1$, we have:
$$
Z_2(t) \geq J(t) \frac{w_1^*/c_1 +w_a^*/c_a}{1+\epsilon} I_{\frac{w_1^*/c_1}{w_1^*/c_1+w_2^*/c_2}}\left(\mu_1, \mu_2\right)=\frac{J(t)}{(1+\epsilon) T^*(\boldsymbol{\mu})}.
$$
Notice that the algorithm will end no later than arm $2$ is eliminated. Consequently, we can bound the cumulative cost a stopping time $\tau_\delta$ as follows:
$$
\begin{aligned}
J(\tau_\delta)  \leq & J \Big(\inf \{t \in \mathbb{N}: Z_2(t) \geq \beta(t, \delta)\}\Big) 
 \leq J(t_1) \vee J\Big(\inf \left\{t \in \mathbb{N}: J(t)(1+\epsilon)^{-1} T^*(\boldsymbol{\mu})^{-1} \geq \log (Bt^\alpha / \delta)\right\}\Big)\\
 \leq & J(t_1) \vee (1+\epsilon)T^*(\boldsymbol{\mu})\left(\log \left(\frac{B}{\ell^\alpha\delta} \right) + \alpha \log(J(\tau_\delta)) \right) + \mathcal{O}(1),
\end{aligned}
$$
for some positive constant $B$ from Proposition ~\ref{prop:beta-co}. Using Lemma \ref{lemma:tech-bound}, it follows that on $\mathcal{E}$, for $\alpha \in[1, e / 2]$
$$
J(\tau_\delta) \leq J(t_1) \vee \alpha(1+\epsilon) T^*(\boldsymbol{\mu})\left[\log \left(\frac{B e\left((1+\epsilon) T^*(\boldsymbol{\mu})\right)^\alpha}{\ell^{\alpha} \cdot \delta}\right)+\log \log \left(\frac{B\left((1+\epsilon) T^*(\boldsymbol{\mu})\right)^\alpha}{\ell^{\alpha} \cdot \delta}\right)\right] + \mathcal{O}(1).
$$
Thus we have:
$$
\limsup _{\delta \rightarrow 0} \frac{J(\tau_\delta)}{\log (1 / \delta)} \leq(1+\epsilon) \alpha T^*(\boldsymbol{\mu}) = \frac{2(1+\epsilon)\alpha\left(\sqrt{c_1}+\sqrt{c_2}\right)^2}{(\mu_1 - \mu_2)^2}.
$$
Letting $\epsilon$ go to zero concludes the proof.
\end{proof}

\subsection{Proof of Lemma.~\ref{lemma:asconv-co}}
\begin{proof}[\textbf{Proof of Lemma.~\ref{lemma:asconv-co}}]
Under the Chernoff-Overlap sampling rule, we will pull arm $2$ at time $t$ when the following condition is satisfied:
\begin{align*}
  \sqrt{\widehat{c}_2(t)} N_2(t) \leq \sqrt{\widehat{c}_1(t)} N_1(t)
.\end{align*}
This gives us the condition:
\begin{align*}
  \left( \frac{\widehat{c}_2(t)}{\widehat{c}_1(t)} \right)^{\frac{1}{2}} N_2(t) \leq N_1(t)
.\end{align*}
Set $c(t) = \left( \frac{\widehat{c}_2(t)}{\widehat{c}_1(t)} \right)^{\frac{1}{2}}$. It is clear that $c(t)$ can be upper and lower bounded by constants as the cost is bounded. Then we pull arm 1 whenever $N_1(t) < c(t)N_2(t)$. From the above condition, it is clear that as $t \to \infty$ it must be the case that $N_a(t) \to \infty$ since $c(t)$ is both upper and lower bounded. Additionally, we have that 
\begin{align*}
    c(t)N_2(t) - c(t) \leq N_1(t)  \leq c(t)N_2(t) + 1.
\end{align*}
This gives us that
\begin{align*}
    \frac{\widehat{c}_1(t) c(t) N_2(t)}{J(t)} - \frac{c(t)}{J(t)} < \frac{\widehat{c}_1(t) N_1(t)}{J(t)} < \frac{\widehat{c}_1(t) c(t) N_2(t)}{J(t)} + \frac{1}{J(t)}.
\end{align*}
Since $\boldsymbol{c} \in \mathcal{C}$ which has bounded distribution, it must be the case that $c(t) \leq 1/\ell^2$ due to the boundedness assumption. Additionally, we have $J(t) \geq \ell t$. Therefore exists some $t_0$ such that for every $t \geq t_0$,
\begin{align*}
    \frac{\max\{1, c(t)\}}{J(t)} < \epsilon/2.
\end{align*}
It follows that for all $t \geq t_0$,
\begin{align}\label{ineq:triang-1}
    \left|\frac{\widehat{c}_1(t) N_1(t)}{J(t)} - \frac{\widehat{c}_1(t) c(t) N_2(t)}{J(t)}\right| < \epsilon/2.
\end{align}
Now note that can upper and lower bound $J(t)$ as follows:
\begin{align*}
    J(t) =& \hat{c}_1(t) N_1 (t) + \hat{c}_2(t) N_2 (t) \leq \widehat{c}_1(t)(c(t) N_2(t) + 1) + \hat{c}_2(t) N_2 (t),\\
    J(t) = &\hat{c}_1(t) N_1 (t) + \hat{c}_2(t) N_2 (t) \geq \widehat{c}_1(t)(c(t) N_2(t) - c(t)) + \hat{c}_2(t) N_2 (t).\\
\end{align*}
Therefore, we have:
\begin{equation}\label{eq:co-conv}
    \frac{\widehat{c}_1(t) c(t) N_2(t)}{\widehat{c}_1(t)(c(t) N_2(t) + 1) + \widehat{c}_2(t) N_2(t)} \leq \frac{\widehat{c}_1(t) c(t) N_2(t)}{J(t)} \leq \frac{\widehat{c}_1(t) c(t) N_2(t)}{\widehat{c}_1(t)( c(t) N_2(t) - c(t)) + \widehat{c}_2(t) N_2(t)}
\end{equation}
Let $\mathcal{E} = \{\widehat{\boldsymbol{c}} \to \boldsymbol{c}\}$, so that $\mathbb{P}(\mathcal{E}) = 1$ by the strong law of large numbers. Then on $\mathcal{E}$, taking the limit of \eqref{eq:co-conv} as $t\to \infty$ yields
\begin{align*}
    \frac{\sqrt{c_1}}{\sqrt{c_1} + \sqrt{c_2}} \leq \lim_{t \to \infty} \frac{\widehat{c}_1 c N_2(t)}{J(t)} \leq \frac{\sqrt{c_1}}{\sqrt{c_1} + \sqrt{c_2}}
\end{align*}
This gives us that there exists some $t_1 \geq t_0$ such that for every $t \geq t_1$
\begin{align}\label{ineq:triang-2}
    \left| \frac{\widehat{c}_1 c N_2(t)}{J(t)} - \frac{\sqrt{c_1}}{\sqrt{c_1} + \sqrt{c_2}}\right| < \epsilon/2
\end{align}
Thus, combining (\ref{ineq:triang-1}) and (\ref{ineq:triang-2}), for every $t \geq t_1$
\begin{align*}
    \left|\frac{\widehat{c}_1 N_1(t)}{J(t)} - \frac{\sqrt{c_1}}{\sqrt{c_1} + \sqrt{c_2}}\right| &=
    \left|\frac{\widehat{c}_1 N_1(t)}{J(t)} - \frac{\widehat{c}_1 c N_2(t)}{J(t)} + \frac{\widehat{c}_1 c N_2(t)}{J(t)} - \frac{\sqrt{c_1}}{\sqrt{c}_1 + \sqrt{c_2}}\right| \\
    &\leq
    \left|\frac{\widehat{c}_1 N_1(t)}{J(t)} - \frac{\widehat{c}_1 c N_2(t)}{J(t)}\right| +
    \left| \frac{\widehat{c}_1 c N_2(t)}{J(t)} - \frac{\sqrt{c}_1}{\sqrt{c_1} + \sqrt{c_2}}\right| 
    < \epsilon/2 + \epsilon/2 = \epsilon
\end{align*}
        Since $\frac{\widehat{c}_2 N_2(t)}{J(t)} = 1 - \frac{\widehat{c}_1 N_1(t)}{J(t)}$, we also have that $\frac{\widehat{c}_2 N_2(t)}{J(t)} \to \frac{\sqrt{c_2}}{\sqrt{c_1} + \sqrt{c_2}}$ as desired.
\end{proof}

\section{$\delta$-PAC Analysis for $\mathsf{CTAS}$ and $\mathsf{CO}$}
In this section, we provide a proof for the $\delta$-PAC guarantees of both algorithms following the identical procedure from \cite{Garovoer16BAI} for completeness. However, we only present the proof for the case where $\alpha>1$. The case of $\alpha = 1$ can also be generalized using the same argument as Theorem 10 in~\cite{Garovoer16BAI}. 
\begin{proof}[\textbf{Proof of Proposition~\ref{prop:beta}}]
The proof relies on the fact that $Z_{a, b}(t)$ can be expressed using function $I_\alpha$ introduced in Definition \eqref{eq:jens-shan}. An interesting property of this function, that we use below, is the following. It can be checked that if $x>y$,
\[
I_\alpha(x, y)=\inf _{x^{\prime}<y^{\prime}}\left[\alpha d\left(x, x^{\prime}\right)+(1-\alpha) d\left(y, y^{\prime}\right)\right] .
\]
For every $a, b$ that are such that $\mu_a<\mu_b$ and $\hat{\mu}_a(t)>\hat{\mu}_b(t)$, one has the following inequality:
\begin{align*}
Z_{a, b}(t) & =\left(N_a(t)+N_b(t)\right) I \frac{N_a(t)}{N_a(t)+N_b(t)}\left(\hat{\mu}_a(t), \hat{\mu}_b(t)\right) \\
& =\inf _{\mu_a^{\prime}<\mu_b^{\prime}} N_a(t) d\left(\hat{\mu}_a(t), \mu_a^{\prime}\right)+N_b(t) d\left(\hat{\mu}_b(t), \mu_b^{\prime}\right) \\
& \leq N_a(t) d\left(\hat{\mu}_a(t), \mu_a\right)+N_b(t) d\left(\hat{\mu}_b(t), \mu_b\right) .
\end{align*}
One has
\begin{align*}
\mathbb{P}_{\boldsymbol{\mu}}\left(\tau_\delta<\infty,\right. & \left.\hat{a}_{\tau_\delta} \neq a^*\right) \leq \mathbb{P}_{\boldsymbol{\mu}}\left(\exists a \in \mathcal{A} \backslash a^*, \exists t \in \mathbb{N}: \hat{\mu}_a(t)>\hat{\mu}_{a^*}(t), Z_{a, a^*}(t)>\beta(t, \delta)\right) \\
& \leq \mathbb{P}_{\boldsymbol{\mu}}\left(\exists t \in \mathbb{N}: \exists a \in \mathcal{A} \backslash a^*: N_a(t) d\left(\hat{\mu}_a(t), \mu_a\right)+N_{a^*}(t) d\left(\hat{\mu}_{a^*}(t), \mu_{a^*}\right) \geq \beta(t, \delta)\right) \\
& \leq \mathbb{P}_{\boldsymbol{\mu}}\left(\exists t \in \mathbb{N}: \sum_{a=1}^K N_a(t) d\left(\hat{\mu}_a(t), \mu_a\right) \geq \beta(t, \delta)\right) \\
& \leq \sum_{t=1}^{\infty} e^{K+1}\left(\frac{\beta(t, \delta)^2 \log (t)}{K}\right)^K e^{-\beta(t, \delta)} .
\end{align*}
The last inequality follows from a union bound and \cite[Theorem 2]{magureanu2014lipschitz}, originally stated for Bernoulli distributions but whose generalization to one-parameter exponential families is straightforward. Hence, with an exploration rate of the form $\beta(t, \delta)=\log \left(B t^\alpha / \delta\right)$, for $\alpha>1$, choosing $B$ satisfying
\[
\sum_{t=1}^{\infty} \frac{e^{K+1}}{K^K} \frac{\left(\log ^2\left(B t^\alpha\right) \log t\right)^K}{t^\alpha} \leq B
\]
yields a probability of error upper bounded by $\delta$.
\end{proof}

\begin{proof}[\textbf{Proof of Proposition~\ref{prop:beta-co}}]
In $\mathsf{CO}$, the event where we incorrectly identify the best arm corresponds to incorrectly eliminating the best arm at some point. Therefore,
we once again have 
\begin{align*}
\mathbb{P}_{\boldsymbol{\mu}}\left(\tau_\delta<\infty,\right. & \left.\hat{a}_{\tau_\delta} \neq a^*\right) \leq \mathbb{P}_{\boldsymbol{\mu}}\left(\exists t \in \mathbb{N}: \exists a \in \mathcal{A} \backslash a^*: N_a(t) d\left(\hat{\mu}_a(t), \mu_a\right)+N_{a^*}(t) d\left(\hat{\mu}_{a^*}(t), \mu_{a^*}\right) \geq \beta(t, \delta)\right) \\
& \leq \mathbb{P}_{\boldsymbol{\mu}}\left(\exists t \in \mathbb{N}: \sum_{a=1}^K N_a(t) d\left(\hat{\mu}_a(t), \mu_a\right) \geq \beta(t, \delta)\right) \\
& \leq \sum_{t=1}^{\infty} e^{K+1}\left(\frac{\beta(t, \delta)^2 \log (t)}{K}\right)^K e^{-\beta(t, \delta)} .
\end{align*}
The result then follows from the same line of reasoning as Proposition~\ref{prop:beta} afterwards.
\end{proof}

\begin{remark}
    For the proof of $\alpha = 1$, see \cite[Theorem 10]{Garovoer16BAI}. The same proof can be used to show that both $\mathsf{CTAS}$ and $\mathsf{CO}$ are $\delta$-PAC for $\alpha = 1$, where $B = 2K$.
\end{remark}

\end{document}